\documentclass{article} 
\usepackage{iclr2027_conference,times}

\usepackage{amsmath,amsfonts,bm}

\def\eqref#1{equation~\ref{#1}}

\def\1{\bm{1}}

\DeclareMathAlphabet{\mathsfit}{\encodingdefault}{\sfdefault}{m}{sl}
\SetMathAlphabet{\mathsfit}{bold}{\encodingdefault}{\sfdefault}{bx}{n}

\usepackage{hyperref}
\usepackage{url}
\usepackage{hyperref}       
\usepackage{url}            
\usepackage{booktabs}       
\usepackage{amsfonts}       
\usepackage{nicefrac}       
\usepackage{microtype}      
\usepackage{xcolor}         
\usepackage{amsmath}
\usepackage{mathrsfs}
\usepackage{amsfonts}
\usepackage{amsmath, amsthm, amsfonts}
\usepackage{booktabs}
\usepackage{multirow}
\usepackage{graphicx}
\usepackage{bm}
\usepackage{amsmath}
\usepackage{amsthm}
\usepackage{wrapfig}
\usepackage{booktabs}
\usepackage{xcolor}
\usepackage{array}

\def\*#1{\mathbf{#1}} 
\def\+#1{\mathcal{#1}} 
\def\-#1{\mathrm{#1}} 
\def\=#1{\boldsymbol{#1}} 
\def\!#1{\mathtt{#1}}
\def\@#1{\mathscr{#1}}

\allowdisplaybreaks

\newtheorem{theorem}{Theorem}[section]

\usepackage{thmtools}

\title{SDDBMs: Soft Denoising Diffusion Bridge Models}

\author{
\textbf{Shiyi Qi} \\
Nanjing University \\
\texttt{syqi981125@163.com}
\And
\textbf{Kun He}\thanks{Corresponding author.} \\
Renmin University of China \\
Nanjing University\\
\texttt{hekun2023@ruc.edu.cn}
\And
\textbf{Mingmou Liu}\thanks{Corresponding author.} \\
Nanjing University\\
\texttt{lmm@nju.edu.cn}
}

\iclrfinalcopy 
\begin{document}

\maketitle
\lhead{}

\begin{abstract}
Diffusion bridge models leverage Doob's \(h\)-transform to construct
stochastic transports between arbitrary endpoint distributions, and have
shown strong potential in image-to-image translation and restoration.
However, most existing bridge models rely on hard endpoint conditioning,
which forces the terminal state to match a prescribed target exactly. This
hard constraint induces terminal-boundary singularities: the terminal law
collapses to a Dirac measure, and the resulting drift coefficients become
ill-conditioned near the endpoint. In this paper, we propose Soft
Denoising Diffusion Bridge Models (SDDBMs), a generalized framework that
regularizes diffusion bridges directly at the level of their terminal
constraints. Instead of imposing an exact endpoint, SDDBMs prescribe a
non-degenerate Gaussian terminal marginal under the transformed path
measure, with a flexible terminal center and variance. Starting from this prescribed marginal, we develop a complete closed-form
construction of the soft bridge, including the Gaussian terminal
reweighting and soft \(h\)-function, the induced Gaussian forward
marginals and \(\mathbf{x}_0\)-free dynamics.
Theoretically, SDDBMs provide a unified probabilistic
perspective that encompasses existing diffusion bridge models, including
DDBMs, GOUB, and UniDB, as special cases under specific parameter choices.
Extensive experiments on image restoration tasks demonstrate that SDDBMs
achieve improved numerical stability and superior generation quality over
existing bridge-based methods.
\end{abstract}

\section{Introduction}
Diffusion models have emerged as a powerful class of generative models,
achieving remarkable success in image synthesis, restoration, video
generation, and other domains~\citep{sohl2015deep,ho2020denoising,song2021scorebased,dhariwal2021diffusion,rombach2022high,podell2024sdxl}. 
Their standard formulation gradually transports data to a simple Gaussian
prior and learns the reverse denoising process. While highly effective for unconditional generation, this paradigm is less natural for tasks that require transport between two structured distributions, such as
image-to-image translation and restoration. In such settings, both
endpoints carry semantic and structural information, and treating one side as pure Gaussian noise may discard the intrinsic correspondence between the source and target domains.

Diffusion bridge models provide a principled alternative by constructing a stochastic process that connects two endpoint distributions. Denoising
Diffusion Bridge Models (DDBMs)~\citep{zhou2024denoising} apply Doob's
\(h\)-transform to build bridges between paired source and target domains, enabling direct distribution-to-distribution transport. GOUB
~\citep{yue2023image} further extends this idea by combining Doob's
\(h\)-transform with a generalized Ornstein--Uhlenbeck process, whose
mean-reverting structure is particularly suitable for image restoration.
These bridge-based models have demonstrated strong empirical performance
on image-to-image tasks by explicitly incorporating endpoint information
into the diffusion dynamics.

Despite their effectiveness, most existing Doob's \(h\)-transform bridges
rely on hard endpoint conditioning: the terminal state of the process,
denoted by \(\mathbf{x}_T\), is forced to coincide with a prescribed target
endpoint \(\mathbf{x}^{\star}\).
This hard
constraint is intuitive, but it also leads to terminal-boundary
singularities. For example, the bridge reverse process in DDBMs is only
defined up to \(T-\epsilon\), and practical sampling requires approximating
the terminal state near \(T\) rather than starting exactly from the
terminal boundary. 
Although this endpoint singularity has been
observed in existing bridge formulations, its probabilistic origin remains
insufficiently understood. In particular, it is still unclear from the
viewpoint of terminal distributions why such singularities arise, whether
they can be avoided, and what form of relaxation is required to remove
them.

A closely related work is UniDB~\citep{zhu2025unidb}, which revisits
diffusion bridges from the perspective of stochastic optimal control (SOC). By
introducing a finite terminal penalty coefficient, UniDB relaxes the
infinitely strict endpoint matching of Doob's \(h\)-transform and derives a
stabilized analytical control term, with hard endpoint bridges recovered as the limiting case when the terminal penalty tends to infinity.
However, UniDB is primarily derived
from SOC, rather than by explicitly specifying a
smooth terminal reweighting under an \(h\)-transformed path measure. 
Consequently, while it has the potential to mitigate singularities,
the probabilistic picture underlying UniDB is not very clear, and
its score parameterization retains a terminal singularity requiring
careful boundary treatment.


In this paper, we propose Soft Denoising Diffusion Bridge Models (SDDBMs), schematically illustrated in Figure~\ref{fig:framework},
to address the terminal singularities of hard-constrained diffusion
bridges at the level of terminal marginals.
Our key insight is that a finite-coefficient, \(\mathbf{x}_0\)-free affine
bridge dynamics necessarily induces a smooth, bounded, non-degenerate
Gaussian terminal distribution that 
retains dependence on the source
state \(\mathbf{x}_0\). Therefore, exact source forgetting or hard endpoint
collapse can only be achieved through a blow-up of the drift
coefficients.
To avoid this numerical singularity at its source, we replace the hard
endpoint constraint with a soft terminal constraint. Specifically, instead
of enforcing \(\mathbf{x}_T=\mathbf{x}^{\star}\) exactly, we prescribe the
terminal marginal under the transformed path measure as a
non-degenerate Gaussian 
$\mathbf{x}_T
\sim
\mathcal N\left(
\mathbf{x}^{\ast},
\sigma^2\mathbf I
\right)$,
where the auxiliary terminal center
\(\mathbf{x}^{\ast}\) 
may differ
slightly from the desired target \(\mathbf{x}^{\star}\) and retain a
controlled dependence on \(\mathbf{x}_0\).
Our theoretical analysis shows that this prescribed terminal marginal can
be exactly realized by Doob's \(h\)-transform through a suitable Gaussian
terminal reweighting. This construction yields a closed-form soft
\(h\)-function, a nonsingular analytical drift correction, and closed-form
Gaussian forward marginals. It further enables a well-posed
\(\mathbf{x}_0\)-free reformulation of the bridge dynamics. These
ingredients form a complete training and sampling framework for SDDBMs,
while also providing a unified probabilistic perspective on existing
diffusion bridge models, including DDBMs, GOUB, and UniDB.

Our main contributions are summarized as follows:

\begin{itemize}

\item We prove that a finite-coefficient, \(\mathbf{x}_0\)-free affine bridge
dynamics necessarily induces a smooth, bounded, non-degenerate Gaussian
terminal distribution that still retains dependence on the source state
\(\mathbf{x}_0\). Consequently, within this class of dynamics, exact source
forgetting or hard endpoint collapse requires the drift coefficients to
blow up near the terminal time. This characterizes the numerical
singularities of hard endpoint bridges from the perspective of their
terminal distributions.

\item To overcome the singularities of hard endpoint bridges, we introduce soft terminal constraints by prescribing a non-degenerate Gaussian terminal
marginal. Starting from this prescribed marginal, we derive the
corresponding Gaussian \(h\)-function in closed form, which further yields
closed-form forward marginals. These
results form the basis of the complete SDDBMs framework, including both the
original \(\mathbf x_0\)-conditioned soft bridge and its
\(\mathbf x_0\)-free reformulation.

\item  We show that the proposed SDDBMs framework provides a unified
perspective on several existing diffusion bridge models, including DDBMs,
GOUB, and UniDB. In particular, although UniDB is derived from the
perspective of stochastic optimal control, our framework gives it an
explicit probabilistic interpretation in terms of Gaussian terminal
reweighting.

\item Our model achieves state-of-the-art performance on various image-to-image tasks, including deraining, super-resolution, and inpainting.
\end{itemize}

\section{Preliminaries}
\label{eq:preliminaries}

\subsection{\texorpdfstring{Doob's $h$-transform}{Doob's h-transform}}
\label{sec:doob_h_transform}
Doob's $h$-transform~\citep{sarkka2019applied,leonard2011stochastic} 
is a classical technique for constructing conditioned or reweighted diffusion processes via a {positive harmonic function}. Let us consider the reference diffusion process governed by the following stochastic differential equation:
\begin{equation}
    \mathrm{d}\mathbf{x}_t
    =
    \mathbf{f}(\mathbf{x}_t,t) \mathrm{d}t
    +
    g_t \mathrm{d}\mathbf{w}_t,
    \qquad 0 \le t \le T,
\end{equation}
where $\mathbf{x}_t \in \mathbb{R}^d$, $\mathbf{f}: \mathbb{R}^d \times [0,T] \rightarrow \mathbb{R}^d$ represents the drift coefficient, $g: [0,T] \rightarrow \mathbb{R}$ is the scalar diffusion coefficient, and $\mathbf{w}_t$ denotes a standard $d$-dimensional Wiener process. 
We assume that the reference process admits a well-defined transition density $p(\mathbf{x}_t,t \mid \mathbf{x}_s,s)$ for $0 \le s < t \le T$.

For a given positive harmonic function $h$, Doob's $h$-transform reweights the reference marginal as:
\begin{equation}
    p_h(\mathbf{x},t \mid \mathbf{x}_0,0)
    =
    \frac{
    p(\mathbf{x},t \mid \mathbf{x}_0,0)h(\mathbf{x},t)
    }
    {
    h(\mathbf{x}_0,0)
    }.
    \label{eq:general_h_marginal}
\end{equation}

Provided that $h$ is sufficiently smooth, the transformed diffusion is governed by:
\begin{equation}
    \mathrm{d}\mathbf{x}_t
    =
    \left[
    \mathbf{f}(\mathbf{x}_t,t)
    +
    g_t^2\nabla_{\mathbf{x}_t}\log h(\mathbf{x}_t,t)
    \right] \mathrm{d}t
    +
    g_t \mathrm{d}\mathbf{w}_t,
    \qquad t < T.
    \label{eq:general_h_sde}
\end{equation}
To encode a soft terminal constraint, we introduce a nonnegative terminal weight 
$q:\mathbb{R}^d \to [0,\infty)$. For $0 \leq t <T$, we define the $h$-function as its conditional expectation, with boundary condition $h(\mathbf{x},T)=q(\mathbf{x})$. 
Substituting this $h$-function into Eq.~(\ref{eq:general_h_marginal}) yields the modified terminal marginal density. This definition and its corresponding density are expressed jointly as:

\begin{equation}
    h(\mathbf{x},t)
    =
    \int_{\mathbb{R}^d}
    p(\mathbf{z},T \mid \mathbf{x},t)q(\mathbf{z}) \, \mathrm{d}\mathbf{z}, \qquad 
    p_h(\mathbf{x},T \mid \mathbf{x}_0,0)
    =
    \frac{
    p(\mathbf{x},T \mid \mathbf{x}_0,0)q(\mathbf{x})
    }
    {
    h(\mathbf{x}_0,0)
    }.
    \label{eq:soft_h_function_terminal_density}
\end{equation}

Consequently, a soft terminal constraint does not arbitrarily replace the reference terminal distribution; rather, it reweights the reference terminal density in an absolutely continuous manner. Throughout, for notational simplicity, we suppress the dependence of $h$ and $p_h$ on the prescribed terminal constraint.


In contrast, a hard endpoint bridge imposes the singular constraint
$\mathbf{x}_T=\mathbf{x}^\star$,
formally corresponding to the Dirac weight
$q(\mathbf{z})=\delta(\mathbf{z}-\mathbf{x}^\star)$.
For $t<T$, the hard $h$-function becomes
$h_{\mathrm{hard}}(\mathbf{x},t)
=p(\mathbf{x}^\star,T\mid\mathbf{x},t)$,
yielding the time-$t$ marginal
$p_{h_{\mathrm{hard}}}
(\mathbf{x},t\mid\mathbf{x}_0,0)
=
\frac{
p(\mathbf{x},t\mid\mathbf{x}_0,0)
p(\mathbf{x}^\star,T\mid\mathbf{x},t)
}{
p(\mathbf{x}^\star,T\mid\mathbf{x}_0,0)
}$, $0<t<T$.
As $t\to T$, this marginal degenerates to
$\delta(\mathbf{x}_T-\mathbf{x}^\star)$.

\subsection{Generalized Ornstein-Uhlenbeck Bridge}
\label{sec:goub}
The Generalized Ornstein--Uhlenbeck (GOU) process is a flexible
mean-reverting diffusion process that extends the classical
Ornstein--Uhlenbeck process to time-varying coefficients~\citep{ahmad1988introduction}: 
\begin{equation}
    \mathrm{d}\mathbf{x}_t
    =
    \theta_t(\boldsymbol{\mu}-\mathbf{x}_t)\,\mathrm{d} t
    +
    g_t\,\mathrm{d}\mathbf{w}_t
\end{equation}
where \(\boldsymbol{\mu}\in\mathbb{R}^d\) is the state vector,
$\theta_t >0$ and $g_t$ are the drift and diffusion coefficients, and \(\mathbf{w}_t\) is a standard \(d\)-dimensional
Wiener process. Its affine drift yields the closed-form transition density $
p(\mathbf{x}_t,t\mid \mathbf{x}_s,s)
=
\mathcal{N}
\left(
\mathbf{x}_t;
\bar{\boldsymbol{\mu}}_{s:t},
\bar{\sigma}_{s:t}^2\mathbf I
\right)$, $0\leq s \leq t$, where

\begin{equation}
    \bar{\boldsymbol{\mu}}_{s:t}
    =
    e^{-\bar{\theta}_{s:t}}\mathbf{x}_s
    +
    \boldsymbol{\mu}
    \left(
    1-e^{-\bar{\theta}_{s:t}}
    \right),
    \qquad \bar{\sigma}_{s:t}^2
    =
    \int_s^t
    e^{-2\bar{\theta}_{u:t}}
    g_u^2\,du,\qquad
    \bar{\theta}_{s:t}
    :=
    \int_s^t\theta_u\,du.
    \label{eq:gou_transition_variance_mean}
\end{equation}
Under the common parametrization
\(g_t^2=2\lambda^2\theta_t\) with \(\lambda>0\), the
transition variance reduces to
$   \bar{\sigma}_{s:t}^2
    =
    \lambda^2
    \left(
    1-e^{-2\bar{\theta}_{s:t}}
    \right)$.
Consequently, as $\bar{\theta}_{0:t}\to\infty$,
the marginal converges to
$\mathcal{N}(\boldsymbol{\mu},\lambda^2 \mathbf{I})$,
with limiting covariance $\lambda^2 \mathbf{I}$.  
A derivation is provided in Appendix~\ref{sec:gou_p}.

The Generalized Ornstein-Uhlenbeck Bridge (GOUB) \citep{yue2023image} 
applies Doob's $h$-transform to the GOU process, serving as a generalization of DDBMs \citep{zhou2024denoising}. Despite its remarkable performance across various image-to-image 
tasks, its reverse-time dynamics suffer from terminal singularities.
Specifically, the reverse-time SDE of GOUB is governed by the following equation:
\begin{equation}
    \mathrm{d} \mathbf{x}_t = \left[ \left( \theta_t + g_t^2 \frac{e^{-2\bar{\theta}_{t:T}}}{\bar{\sigma}^2_{t:T}} \right) (\mathbf{x}^{\star} - \mathbf{x}_t) - g_t^2 \nabla_{\mathbf{x}_t} \log {p_h(\mathbf{x}_t,t)} \right] \mathrm{d} t + g_t \mathrm{d}\bar{\mathbf w}_t.
    \label{eq:goub_reverse}
\end{equation}

However, at the terminal boundary $t=T$, Eq.~(\ref{eq:goub_reverse}) suffers from severe numerical singularities. 
First, 
the linear analytical drift diverges because the fraction $\frac{e^{-2\bar{\theta}_{t:T}}}{\bar{\sigma}^2_{t:T}}$ explodes 
as 
$\bar{\sigma}^2_{t:T} \rightarrow 0$.
Simultaneously, the score function ${\nabla_{\mathbf{x}_t}}\log {p_h(\mathbf{x}_t,t)}$ becomes singular, because the distribution inherently collapses to a Dirac distribution at the terminal time. Consequently, these interacting singularities prevent direct numerical sampling starting exactly from $t=T$.

UniDB~\citep{zhu2025unidb} revisits diffusion bridges from the perspective
of stochastic optimal control and introduces a stabilized reverse-time
process:
\begin{equation}
    \mathrm{d} \mathbf{x}_t
    =
    \left[
    \left(
    \theta_t
    +
    g_t^2
    \frac{
    e^{-2\bar{\theta}_{t:T}}
    }
    {
    \kappa^{-1}+\bar{\sigma}_{t:T}^2
    }
    \right)
    \left(
    \mathbf{x}^{\star}-\mathbf{x}_t
    \right)
    -
    g_t^2
    \nabla_{\mathbf{x}_t} \log {p_h(\mathbf{x}_t,t) }
    \right]\mathrm{d} t
    +
    g_t \mathrm{d}\bar{\mathbf w}_t,
    \label{eq:unidb_reverse}
\end{equation}
where \(\kappa>0\) is the terminal penalty coefficient. For finite \(\kappa\), the
additional \(\kappa^{-1}\) term prevents the analytical control component
from diverging as \(\bar{\sigma}_{t:T}^2\to0\), balancing endpoint accuracy against drift stability. However, UniDB is primarily derived from the perspective of SOC and does not explicitly formulate this finite-penalty
mechanism as a smooth terminal reweighting under Doob's \(h\)-transform, leaving the induced terminal marginal distribution implicit. It is also worth noting that, although $\kappa$ regularizes the analytical drift away from the hard-endpoint singularity, the learned score in UniDB is implemented through a noise-parameterized form 
$\mathbf{s}_\theta(\mathbf{x}_t, \mathbf{x}^{\star}, t) = - \boldsymbol{\epsilon}_\theta(\mathbf{x}_t, \mathbf{x}^{\star}, t) / \bar{\sigma}'_t$, 
where $\mathbf{s}_\theta$ denotes the learned conditional score, $\boldsymbol{\epsilon}_\theta$ is the corresponding noise-prediction network, and $\bar{\sigma}'_t$ is the bridge noise scale used in UniDB.
Since \(\bar{\sigma}'_{t}\) vanishes at the endpoints, the
endpoint behavior of this parameterization still requires careful boundary
treatment in implementation.

\section{Soft-Constrained Diffusion Bridges}
\begin{figure}[ht]
    \centering
    \includegraphics[width=1.0\textwidth]{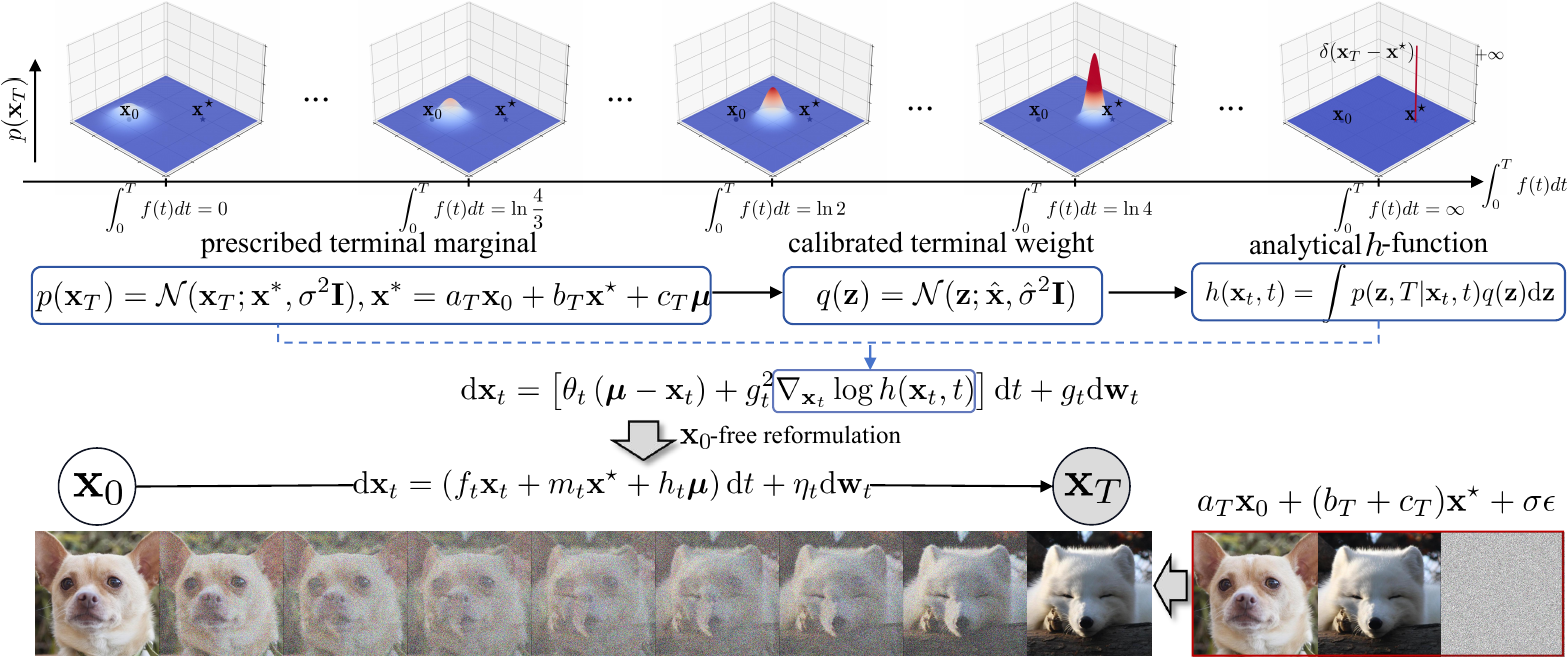}

\caption{Illustration of hard-endpoint singularity and the SDDBMs framework.
For the linear SDE
$\mathrm{d}\mathbf{x}_t
=
f(t)(\mathbf{x}^\star-\mathbf{x}_t)\mathrm{d}t
+
\mathrm{d}\mathbf{w}_t$,
the hard-endpoint limit
$\mathbf{x}_T=\mathbf{x}^\star$
corresponds to
$p(\mathbf{x}_T)\to
\delta(\mathbf{x}_T-\mathbf{x}^\star)$
and requires $f(t)$ to become unbounded as $t\to T$
(see Proposition~\ref{thm:affine_nonsingular_terminal}
for the general result). To avoid this degeneracy,
SDDBMs prescribe a non-degenerate Gaussian terminal marginal
$p(\mathbf{x}_T)$, which determines a calibrated Gaussian terminal weight $q(\mathbf{z})$ and an analytical soft $h$-function via terminal reweighting.
The resulting $h$-transformed bridge admits nonsingular drift corrections
and closed-form Gaussian forward marginals,
and can be further rewritten into an equivalent $x_0$-free formulation
for flexible reverse-time parameterization. The terminal-state illustration uses
$\boldsymbol{\mu}=\mathbf{x}^{\star}$, as in our experiments.
}
    \label{fig:framework}
\end{figure}
\subsection{Terminal Regularity of Finite-Coefficient Affine Bridges}

We now ask what kind of terminal distributions can be generated by a finite-time
bridge whose drift remains regular up to the terminal time. To make this question concrete, we consider the $x_0$-free affine diffusion
$\mathrm{d}\mathbf{x}_t
=
(A_t\mathbf{x}_t+\mathbf{b}_t)\mathrm{d}t
+
g_t\mathrm{d}\mathbf{W}_t$,
with bounded coefficients and finite nonzero diffusion energy (we relegate the complete setting description, detailed formulations of these assumptions, and their full justifications to Appendix~\ref{sec:full_assumptions_and_justifications}).

Under these assumptions, the natural question is whether a regular $\mathbf{x}_0$-free affine dynamics can produce a hard endpoint distribution or completely forget its dependence on $\mathbf{x}_0$ within a finite time horizon. The following standard result, adapted from classical linear SDE theory~\citep{sarkka2019applied, oksendal2013stochastic}, shows that neither is possible: the terminal law remains non-degenerate and retains dependence on $\mathbf{x}_0$. We formalize this result in the following proposition.


\begin{restatable}{proposition}{thmterminaldistribution}
\label{thm:affine_nonsingular_terminal}
Consider the affine diffusion dynamics governed by the SDE $\mathrm{d}\mathbf{x}_t
    =
    (A_t\mathbf{x}_t + \mathbf{b}_t)\mathrm{d}t
    +
    g_t \mathrm{d}\mathbf{W}_t$ on a finite time interval $[0,T]$,
where $A: [0,T] \to \mathbb{R}^{d \times d}$, $\mathbf{b}: [0,T] \to \mathbb{R}^d$, and $g: [0,T] \to \mathbb{R}$ are deterministic, time-dependent coefficients. Notably, for any time $t$, these parameters are strictly independent of the initial state $\mathbf{x}_0$, the current state $\mathbf{x}_t$, and the driving Brownian motion.
Assume that 
\begin{equation}
    \sup_{0\le t\le T}\|A_t\|<\infty,
    \qquad
    \sup_{0\le t\le T}\|\mathbf b_t\|<\infty,\qquad 
0<\int_0^T g_t^2\mathrm{d}t<\infty.
    \label{eq:bounded_affine_coefficients}
\end{equation}
Let \(\Phi(t,s)\) be the fundamental matrix associated with \(A_t\), defined
by
$\Phi(s,s)=\mathbf{I},\frac{\partial}{\partial t}\Phi(t,s)
    =
    A_t\Phi(t,s)$.
Then the terminal law of \(\mathbf x_T\) is a non-degenerate Gaussian:
    $\mathbf x_T\mid \mathbf x_0
    \sim
    \mathcal N
    \left(
    \mathbf m_T(\mathbf x_0),
    \Sigma_T
    \right)$
where
$\mathbf m_T(\mathbf x_0)
    =
    \Phi(T,0)\mathbf x_0
    +
    \int_0^T \Phi(T,s)\mathbf b_s\, \mathrm{d}s,
    \Sigma_T
    =
    \int_0^T
    g_s^2\Phi(T,s)\Phi(T,s)^\top \,\mathrm{d}s$.
Here, $\Phi(T,0)$ is invertible and \(\Sigma_T\) is positive definite.
Consequently, \(p(\mathbf x,T\mid \mathbf x_0,0)\) is a smooth, continuous,
and bounded density. Moreover, 
$\mathbf m_T(\mathbf x_0)\neq \mathbf m_T(\mathbf y_0)$
for any
\(\mathbf x_0\neq\mathbf y_0\).
\end{restatable}

The proof is provided in Appendix~\ref{sec:thm_affine_nonsingular_terminal}. Proposition~\ref{thm:affine_nonsingular_terminal} shows that hard endpoint bridges become singular because they require both terminal collapse to a fixed endpoint and complete forgetting of $\mathbf{x}_0$ within a finite time horizon. Further discussion is provided in Appendix~\ref{sec-measure-singularity}.



\subsection{Regularized Diffusion Bridges with Soft Terminal Constraints}

Motivated by Proposition~\ref{thm:affine_nonsingular_terminal}, we regularize
the bridge at the level of the terminal constraint. Rather than enforcing
the endpoint through a degenerate Dirac measure, we prescribe the terminal
marginal of the transformed process as a non-degenerate Gaussian:
\begin{align}
\label{eq-terminal-marginal}
p_{h_{\mathrm{soft}}}(\mathbf{x}_T,T\mid \mathbf{x}_0,0)
=
\mathcal{N}
\left(
\mathbf{x}_T;
\mathbf{x}^{\ast},
\sigma^2\mathbf I
\right).
\end{align}

Here, $\mathbf{x}^\ast$ is a soft terminal center associated with the desired target endpoint $\mathbf{x}^\star$. In contrast to the hard bridge that enforces termination exactly at $\mathbf{x}^\star$, the proposed formulation assigns a smooth terminal density, under which endpoints closer to $\mathbf{x}^\ast$ have higher probability density and those farther away have lower density. Since $\mathbf{x}^\ast$ remains close to $\mathbf{x}^\star$, the soft terminal constraint still guides the process towards the target, while avoiding the singularity of exact endpoint conditioning.

The parameter $\sigma$ controls the strength of the terminal constraint. Smaller $\sigma$ yields a sharper terminal distribution and therefore a stronger preference for endpoints near $\mathbf{x}^\ast$, while any finite $\sigma>0$ keeps the terminal distribution non-degenerate. As $\sigma \to 0$, the terminal marginal converges to a Dirac measure at $\mathbf{x}^\ast$, recovering the hard-endpoint bridge for $\mathbf{x}^\ast=\mathbf{x}^\star$. Thus, the proposed construction continuously relaxes the hard bridge with a finite, smooth terminal penalty.


Allowing $\mathbf{x}^\ast \neq \mathbf{x}^\star$ facilitates the $\mathbf{x}_0$-free reformulation by keeping the drift correction well-conditioned. We discuss the construction of $\mathbf{x}^\ast$ in Appendix~\ref{sec:rho}.



\subsection{\texorpdfstring{Terminal-Guided $h$-Function and Forward Process}{Terminal-Guided h-Function and Forward Process}}

In this section, we derive the soft $h$-function $h_{\mathrm{soft}}(\mathbf{x}_t,t)$ that ensures the transformed process attains the prescribed terminal marginal in Eq.~(\ref{eq-terminal-marginal}). Using this $h$-function, we obtain closed-form forward marginals, which provide explicit drift corrections for the reverse-time SDE.


The following theorem derives the exact form of the terminal weighting in the \(h\)-function directly from the prescribed terminal marginals under the \(h\)-transformed measure.

\begin{restatable}{theorem}{thmphkey}
\label{thm-ph-key}
Consider the GOU process initialized at $\mathbf{x}_0$, with terminal transition density $p(\mathbf x,T\mid \mathbf{x}_0,0)
=
\mathcal N
\left(
\mathbf x;
\bar{\boldsymbol{\mu}}_{0:T},
\bar{\sigma}_{0:T}^2\mathbf I
\right)$, where $\bar{\boldsymbol{\mu}}_{0:T}$ and $\bar{\sigma}_{0:T}^2$ denote the conditional mean and variance of the state at time $T$ given $\mathbf{x}_0$, respectively. Let $q$ be a nonnegative probability density and define 

\begin{align}\label{eq:h_function_new}
h(\mathbf{x},t)
=
\int p(\mathbf z,T\mid \mathbf{x},t)q(\mathbf z)\,\mathrm{d} \mathbf z.
\end{align}
Assume that the terminal marginal under this transformed measure is
$p_h(\mathbf x,T\mid\mathbf{x}_0,0)
=
\mathcal N(\mathbf x;\mathbf{x}^\ast,\sigma^2\mathbf I)$
for some \(0<\sigma^2<\bar{\sigma}_{0:T}^2\). Then \(q\) is uniquely
determined up to almost-everywhere equality and is given by
$q(\mathbf x)
=
\mathcal N(\mathbf x;\hat{\mathbf{x}},\hat{\sigma}^2\mathbf I)$,
where
\begin{align}\label{eq-xhat-sigmahat}
\hat{\mathbf{x}}
=
\frac{
\bar{\sigma}_{0:T}^2
}
{
\bar{\sigma}_{0:T}^2-\sigma^2
}
\mathbf{x}^\ast
-
\frac{
\sigma^2
}
{
\bar{\sigma}_{0:T}^2-\sigma^2
}
\bar{\boldsymbol{\mu}}_{0:T},
\qquad 
\hat{\sigma}^2
=
\frac{
\sigma^2\bar{\sigma}_{0:T}^2
}
{
\bar{\sigma}_{0:T}^2-\sigma^2
}.
\end{align}
\end{restatable}

The proof is provided in Appendix~\ref{sec:thm_ph-key}. Importantly, 
$q$ is a terminal weight rather than the terminal marginal itself. Since the terminal marginal under the
\(h\)-transformed path measure satisfies
$p_h(\mathbf x,T\mid \mathbf{x}_0,0)
\propto
p(\mathbf x,T\mid \mathbf{x}_0,0)
q(\mathbf x)$,
centering \(q\) directly at
\(\mathbf{x}^\ast\) would only yield a precision-weighted interpolation
between the original terminal mean and \(\mathbf{x}^\ast\). To realize the prescribed marginal
$ \mathcal N(\mathbf x;\mathbf{x}^\ast,\sigma^2\mathbf I)$,
we instead calibrate the terminal 
{weight} by requiring
$q_{\hat{\mathbf{x}},\hat{\sigma}}(\mathbf x)
\propto
{\mathcal N(\mathbf x;\mathbf{x}^\ast,\sigma^2\mathbf I)}
/{p(\mathbf x,T\mid \mathbf{x}_0,0)}$.
For Gaussian \(p(\mathbf x,T\mid \mathbf{x}_0,0)\), this 
{weight} is again
Gaussian provided \(0<\sigma^2<\bar{\sigma}_{0:T}^2\), yielding Eq. (\ref{eq-xhat-sigmahat}).
Thus, \(\hat{\mathbf{x}}\) is a calibrated 
{weight} center rather than
the terminal mean.


The following theorem provides a fully explicit closed-form
expression for the soft \(h\)-function defined in
Eq.~(\ref{eq:h_function_new}). 
Furthermore, it demonstrates that this
\(h\)-function satisfies the Kolmogorov Backward Equation (KBE),
thereby linking the \(h\)-transform directly to the dynamics of the diffusion
process.

\begin{restatable}{theorem}{thmhfunction}
\label{thm:h_function}
Consider the transition kernel of the GOU process, given by 
\begin{align}\label{eq-def-pztxtt}
p(\mathbf{z}, T | \mathbf{x}_t, t) = \mathcal{N}(\mathbf{z}; \bar{\boldsymbol{\mu}}_{t:T}, \bar{\sigma}_{t:T}^2 \mathbf{I}),
\end{align}
where $\bar{\boldsymbol{\mu}}_{t:T}$ and $\bar{\sigma}_{t:T}^2$ denote the conditional mean and variance of the state at time $T$ given $\mathbf{x}_t$, respectively. Let $q(\mathbf{z}) := \mathcal{N}(\mathbf{z}; \hat{\mathbf{x}}, \hat{\sigma}^2 \mathbf{I})$ be a Gaussian density with parameters $\hat{\mathbf{x}} \in \mathbb{R}^d$ and $\hat{\sigma} > 0$. We define the function $h(\mathbf{x}_t, t)$ via the 
{integral} $h(\mathbf{x}_t, t) := \int p(\mathbf{z}, T | \mathbf{x}_t, t) q(\mathbf{z}) \,\mathrm{d} \mathbf z$. For any $t \in [0, T]$, $h(\mathbf{x}_t, t)$ admits the closed-form expression:
\begin{equation}
h(\mathbf{x}_t, t) = \frac{1}{[2\pi(\hat{\sigma}^2 + \bar{\sigma}^2_{t:T})]^{d/2}} \exp \left[ -\frac{\|\bar{\boldsymbol{\mu}}_{t:T} - \hat{\mathbf{x}}\|^2}{2(\hat{\sigma}^2 + \bar{\sigma}^2_{t:T})} \right].
\label{eq:h_function_new_trans}
\end{equation}
Furthermore, for any $t \in [0, T)$,
$h(\mathbf{x}_t, t)$ satisfies the Kolmogorov Backward Equation.
\end{restatable}

The proof is provided in Appendix~\ref{sec:thm_h_function}. 
With Eq.~(\ref{eq:h_function_new_trans}) and Doob's $h$-transform, we have a new SDE:
\begin{equation}
    \begin{aligned}
       \mathrm{d}\mathbf{x}_t
        =& \left[\theta_t(\boldsymbol{\mu}-\mathbf{x}_t)+g^2_t e^{-\bar{\theta}_{t:T}}\frac{\hat{\mathbf{x}}-\bar{\boldsymbol{\mu}}_{t:T}}{\hat{\sigma}^2+\bar{\sigma}^2_{t:T}}\right]\mathrm{d} t + g_t \mathrm{d}\mathbf{w}_t.
    \end{aligned}
    \label{eq:new_sde}
\end{equation}

Compared with GOUB~\citep{yue2023image},
whose drift contains the singular denominator \(\bar{\sigma}_{t:T}^2\), our
softened formulation introduces an additional terminal variance
\(\hat{\sigma}^2\). Consequently, when \(\hat{\sigma}^2>0\), the denominator
\(\hat{\sigma}^2+\bar{\sigma}_{t:T}^2\) remains strictly positive even as
\(t\to T\). This prevents the analytical drift correction from exploding at
the terminal boundary. {Compared with UniDB~\citep{zhu2025unidb}, our construction imposes the
terminal relaxation explicitly at the level of the \(h\)-function. }

Furthermore, Eq.~(\ref{eq:new_sde}) provides flexible control over the
Gaussian terminal reweighting parameters \(\hat{\mathbf{x}}\) and
\(\hat{\sigma}^2\). This flexibility yields a more stable and general
diffusion bridge framework.

Recall that, for the GOU process, 
setting $\mathbf{x}^\ast=\mathbf{x}^\star$~\footnote{\label{footnote_1} It is important to emphasize that the restriction $\mathbf{x}^{\ast}
\neq
\mathbf{x}^{\star}$ is specific to the
\(\mathbf{x}_0\)-free reformulation. For training or sampling formulations
that 
condition explicitly on \(\mathbf{x}_0\), one may still
set \(\mathbf{x}^{\ast}=\mathbf{x}^{\star}\). 
}
in Eq.~(\ref{eq-xhat-sigmahat}) gives
\begin{equation}\label{eq-xhat-xzero}
\hat{\mathbf{x}}
=  - \frac{
\sigma^2 e^{-\bar{\theta}_{0:T}}
}
{
\bar{\sigma}_{0:T}^2-\sigma^2
} \mathbf{x}_0  + \frac{
\bar{\sigma}_{0:T}^2
}
{
\bar{\sigma}_{0:T}^2-\sigma^2
}
\mathbf{x}^\star - \frac{
\sigma^2 \left(1-e^{-\bar{\theta}_{0:T}}\right)
}
{
\bar{\sigma}_{0:T}^2-\sigma^2
}\boldsymbol{\mu}.
\end{equation}
Then we can obtain an explicit parametrization of the
soft \(h\)-transformed SDE in Eq.~(\ref{eq:new_sde}), which is free of $\hat{\mathbf{x}}$. 
See Appendix~\ref{sec:explicit_param} for more details.

Given Eq.~(\ref{eq:new_sde}), the forward marginal distributions at any intermediate time \(t\in[0,T]\) can be derived in closed form, yielding the following proposition:

\begin{restatable}{proposition}{propph}
\label{prop:ph}

Consider the stochastic process governed by Eq.~(\ref{eq:new_sde}).
For \(0<t\le T\), its time-\(t\) marginal density under the \(h\)-transformed path measure is
\begin{equation}\label{eq:ph}
p_h(\mathbf{x}_t,t|\mathbf{x}_0,0)=\mathcal N
\left(
\mathbf{x}_t;
\bar{\boldsymbol{\mu}}_{0:t}
+
\frac{
e^{-\bar{\theta}_{t:T}}\bar{\sigma}_{0:t}^2
}
{
\hat{\sigma}^2+\bar{\sigma}_{0:T}^2
}
(\hat{\mathbf{x}}-\bar{\boldsymbol{\mu}}_{0:T}),
\frac{
\bar{\sigma}_{0:t}^2
(\hat{\sigma}^2+\bar{\sigma}_{t:T}^2)
}
{
\hat{\sigma}^2+\bar{\sigma}_{0:T}^2
}
\mathbf I
\right).
\end{equation}
\end{restatable}

The derivation of Proposition~\ref{prop:ph} is provided in Appendix~\ref{sec:prop_ph}.

\subsection{\texorpdfstring{Tractable $\mathbf{x}_0$-Free Reformulation}{Tractable x0-Free Reformulation}}
\label{sec:reformulation}


Inspired by the affine structure of $\hat{\mathbf{x}}$ parametrized by $(\mathbf{x}_0, \mathbf{x}^\star, \boldsymbol{\mu})$ in Eq.~(\ref{eq-xhat-xzero}), we consider a more general and flexible case by setting $\hat{\mathbf{x}} = \alpha \mathbf{x}_0 + \beta \mathbf{x}^\star + \gamma \boldsymbol{\mu}$, where $\alpha, \beta$, and $\gamma$ are time-independent constants. Although certain training or sampling formulations may allow explicit conditioning on $\mathbf{x}_0$, constructing a standard reverse-time generative process where the source state is unobserved during inference typically requires a forward drift free of $\mathbf{x}_0$. By eliminating this $\mathbf{x}_0$ dependence (see Appendix~\ref{sec-remove-xzero} for detailed derivations), we obtain an equivalent forward process driven solely by the current state and target parameters. 
Following~\cite{song2021scorebased}, we have the $\mathbf{x}_0$-free reverse SDE:

\begin{equation}
\mathrm{d}\mathbf{x}_t = \left[f_t \mathbf{x}_t + m_t \mathbf{x}^\star + {h}_t \boldsymbol{\mu} - \eta^2_t \nabla_{\mathbf{x}_t} \log p_h(\mathbf{x}_t,t ) \right]\mathrm{d}t +\eta_t \mathrm{d}\bar{\mathbf{w}}_t,
\label{eq:new_sde_reverse_mean_ode}
\end{equation}
where
\begin{align}
    \phi_t
    &=
    e^{-\bar{\theta}_{0:t}}
    \left(\hat{\sigma}^2+\bar{\sigma}^2_{t:T}\right)
    +
    \alpha e^{\bar{\theta}_{t:T}}
    \left(\bar{\sigma}^2_{0:T}-\bar{\sigma}^2_{t:T}\right),\qquad
    \varphi_t
    =
    \sigma^2
    \left(
    \bar{\sigma}^2_{0:T}
    -
    \bar{\sigma}^2_{t:T}
    \right)
    +
    \bar{\sigma}^2_{t:T}\bar{\sigma}^2_{0:T},\label{eq-def-phit}\\
    f_t
    &=
    g_t^2
    {
    e^{-\bar{\theta}_{t:T}}
    \left(\alpha-e^{-\bar{\theta}_{0:T}}\right)
    }
    \big/{
    \phi_t
    }
    -
    \theta_t,\qquad\qquad\qquad\;
    m_t
    =
    g_t^2
    {
    \beta e^{-\bar{\theta}_{0:T}}
    }
    \big/{
    \phi_t
    }, \label{eq:ft_def}\\
    h_t
    &=
    \theta_t
    +
    {
    g_t^2 e^{-\bar{\theta}_{t:T}}
    \left[
    e^{-\bar{\theta}_{0:T}}
    +(\gamma-1)e^{-\bar{\theta}_{0:t}}
    +\alpha\left(e^{-\bar{\theta}_{0:t}}-1\right)
    \right]
    }
    \Big/{
    \phi_t
    }.
    \label{eq:ht_def}\\
      \eta_t
    &=
    \frac{g_t}{\bar{\sigma}^2_{0:T}}
    \left\{
    \varphi_t
    \left[
    1-
    \frac{
    2e^{-\bar{\theta}_{t:T}}
    \bar{\sigma}^2_{0:t}
    \left(\alpha-e^{-\bar{\theta}_{0:T}}\right)
    }
    {
    \phi_t
    }
    \right]
    +
    \bar{\sigma}^2_{0:t}
    e^{-2\bar{\theta}_{t:T}}
    \left(
    \sigma^2-\bar{\sigma}^2_{0:T}
    \right)
    \right\}^{1/2}.
    \label{eq:eta}
    \end{align}

\subsection{Training objective of SDDBMs}\label{sec-objective}


Following prior diffusion bridge models~\citep{yue2023image,zhu2025unidb}, we parameterize the reverse transition $p_\theta(\mathbf{x}_{t-1},t-1\mid
\mathbf{x}_t,t)$
with a noise-prediction network
$\boldsymbol{\epsilon}_\theta(\mathbf{x}_t,\mathbf{x}^\star,t)$. Training then matches this parameterized reverse transition with the true one-step posterior $p_h(\mathbf{x}_{t-1},t-1\mid
\mathbf{x}_0,0,\mathbf{x}_t,t)$. 
 For these Gaussian reverse kernels, we train by matching their means $\boldsymbol{\mu}_{t-1}$ and $\boldsymbol{\mu}_{\theta,t-1}$:

\[
\mathcal{L}
=
\mathbb{E}_{t,\mathbf{x}_0,\mathbf{x}_t}
\left[
\left\|
\boldsymbol{\mu}_{t-1}
-
\boldsymbol{\mu}_{\theta,t-1}
\right\|_1
\right].
\]
Here, we use the $L_1$ loss and omit the variance-dependent coefficient $1/(2\sigma^2_{\theta,t-1})$ arising from the Gaussian KL objective, consistent with prior works~\citep{yue2023image,zhu2025unidb}. Detailed derivations are provided in Appendix~\ref{sec:training_objective_append}.

\section{SDDBMs unify diffusion bridge models}\label{sec-unified-framework}
In this section, we show that the proposed SDDBMs framework provides a unified perspective that encompasses several existing diffusion bridge models, including DDBMs~\citep{zhou2024denoising}, GOUB~\citep{yue2023image},
and UniDB~\citep{zhu2025unidb}.

\begin{restatable}{proposition}{propgeneralization}
\label{prop:generalization}
The proposed SDDBMs framework recovers DDBM-VP, DDBM-VE,
GOUB, and UniDB as special cases under the hyperparameter choices listed in
Table~\ref{table:hyper_space}.
\end{restatable}

The detailed algebraic verification is provided in
Appendix~\ref{sec:prop_generalization}.
For DDBM-VP, DDBM-VE, and GOUB, which are based on hard constraints, the terminal marginal of the transformed process is effectively a Dirac delta distribution centered at \(\mathbf{x}^{\star}\).
Consequently, these hard-constraint methods can naturally be regarded as special cases of the soft-constraint formulation proposed in this work.

As previously discussed, although UniDB was originally derived from the perspective of stochastic optimal control, it admits a direct probabilistic interpretation under our framework. Specifically, according to the last row of Table~\ref{table:hyper_space}, setting
$\boldsymbol{\mu} = \mathbf{x}^\star$, $\hat{\mathbf{x}} = \alpha \mathbf{x}_0 + \beta \mathbf{x}^\star + \gamma \boldsymbol{\mu} = \mathbf{x}^\star$, $\hat{\sigma}^2 = \kappa^{-1}$ recovers the UniDB dynamics. Moreover, by Eq.~(\ref{eq:ph}), the terminal marginal of the transformed UniDB process is
$
p_h(\mathbf{x}_T,T \mid \mathbf{x}_0,0)
=
\mathcal{N}
\left(
\mathbf{x}_T;
\left(1 - \frac{e^{-\bar{\theta}_{0:T}}}{\kappa\bar{\sigma}_{0:T}^2 + 1}\right)\mathbf{x}^\star + \frac{e^{-\bar{\theta}_{0:T}}}{\kappa\bar{\sigma}_{0:T}^2 + 1} \mathbf{x}_0,
\frac{\bar{\sigma}_{0:T}^2}{\kappa\bar{\sigma}_{0:T}^2 + 1} \mathbf{I}
\right)$. Thus, UniDB corresponds to a soft Gaussian terminal constraint under our framework.

\begin{wraptable}{r}{0.54\textwidth}
    \centering 
    \vspace{-10pt} 
    
    \caption{Hyper-parameter spaces for different diffusion bridge models.}
    
    \vspace{0pt}
    \begin{tabular}{l c c c c c c } 
        \toprule
        $\mathcal{H}$ & $\theta_t$ & $\boldsymbol{\mu}$ & $\alpha$ & $\beta$ & $\gamma$ & $\hat{\sigma}^2$ \\
        \midrule
        DDBMs (VP)   & $\frac{1}{2}g^2_t$    & $\mathbf{0}$ & $0$             & $1$          & $0$        & $0$  \\
        DDBMs (VE)   & $0$    & $\mathbf{0}$    & $0$            & $1$          & $0$       & $0$  \\
        GOUB &   $\frac{1}{2\lambda^2}g^2_t$ & $\mathbf{x}^\star$    & $0 $            & $1$          & $0$        & $0$    \\
        UniDB & $\frac{1}{2\lambda^2} g^2_t$  & $\mathbf{x}^\star$    & $0$            & $1$          & $0$        & $\kappa^{-1}$  \\
        \bottomrule
        \label{table:hyper_space}
    \end{tabular}
    \vspace{-10pt}
\end{wraptable}

While UniDB is recovered as a constrained special case, SDDBMs constitute a more fundamental and flexible diffusion bridge framework rather than a mere probabilistic reinterpretation. Built upon an explicit $h$-transform characterization, the generalized parameters $\alpha$, $\beta$, $\gamma$, and $\hat{\sigma}^2$ provide flexible control over the terminal center and variance, yielding a broader family of diffusion bridge dynamics. This flexibility inherently maintains a non-vanishing terminal noise scale, thereby avoiding the boundary training instability and specialized score parameterization required by UniDB. 

\section{Experiments}

In this section, we evaluate our models on three popular image restoration tasks: image super-resolution, image deraining, and image inpainting. Following~\cite{zhu2025unidb}, we employ four standard evaluation metrics: Peak Signal-to-Noise Ratio (PSNR)~\citep{fardo2016formal}, Structural Similarity Index (SSIM)~\citep{wang2004image}, Learned Perceptual Image Patch Similarity (LPIPS)~\citep{zhang2018unreasonable}, and Fr\'echet Inception Distance (FID)~\citep{heusel2017gans}. We choose Bicubic~\citep{zhu2025unidb}, MAXIM~\citep{tu2022maxim}, MHNet~\citep{gao2025mixed}, PromptIR~\citep{potlapalli2023promptir}, DDRM~\citep{kawar2022denoising}, IR-SDE~\citep{luo2023image}, GOUB~\citep{yue2023image}, and UniDB~\citep{zhu2025unidb} as our baselines for comparison. Further experimental details are provided in Appendix~\ref{sec:imp_details}.

\begin{table*}[ht]
\setlength{\tabcolsep}{2pt}
  \centering
  \vspace{-4mm}
  \caption{Quantitative comparison with relevant baselines on the DIV2K, Rain100H, and CelebA-HQ 
  datasets. The best and second-best results are highlighted in bold and underlined, respectively.}
  \vskip 0.1in
  \renewcommand{\arraystretch}{1.1}
  \resizebox{\textwidth}{!}{
  \begin{tabular}{|c|cccc|c|cccc|c|cccc|}
    \hline
    \multirow{2}*{\textbf{METHOD}} & \multicolumn{4}{c|}{\textbf{Image Deraining}} & \multirow{2}*{\textbf{METHOD}} & \multicolumn{4}{c|}{\textbf{Image Super-Resolution}} & \multirow{2}*{\textbf{METHOD}} & \multicolumn{4}{c|}{\textbf{Image Inpainting}} \\
    \cline{2-5} \cline{7-10} \cline{12-15}
      & \textbf{PSNR}$\uparrow$ & \textbf{SSIM}$\uparrow$ & \textbf{LPIPS}$\downarrow$ & \textbf{FID}$\downarrow$ & &\textbf{PSNR}$\uparrow$ & \textbf{SSIM}$\uparrow$ & \textbf{LPIPS}$\downarrow$ & \textbf{FID}$\downarrow$ & & \textbf{PSNR}$\uparrow$ & \textbf{SSIM}$\uparrow$ & \textbf{LPIPS}$\downarrow$ & \textbf{FID}$\downarrow$ \\
    \hline
    MAXIM & 30.81 & 0.902 & 0.133 & 58.72 & Bicubic & 26.70 & 0.774 & 0.425 & 36.18 &  PromptIR & 30.22 & 0.918 & 0.068 & 32.69\\
    MHNet &31.08 & 0.899 & 0.126 & 57.93 & DDRM & 24.35 & 0.592 & 0.364 & 78.71 & DDRM & 27.16 & 0.899 & 0.089 & 37.02 \\
    IR-SDE & 31.65 & 0.904 & 0.047 & 18.64 & IR-SDE &  25.90 & 0.657 & 0.231 & 45.36 & IR-SDE & 28.37 & 0.916 & 0.046 & 25.13 \\
    GOUB (SDE) & 31.96 & 0.9028 & 0.046 & 18.14 & GOUB (SDE) & 26.89 & 0.7478 & 0.220 & 20.85 & GOUB (SDE) & 28.98 & 0.9067 & 0.037 & 4.30 \\
    GOUB (ODE) & 34.56 & 0.9414 & 0.077 & 32.83 & GOUB (ODE) & 28.50 & 0.8070 & 0.328 & 22.14 & GOUB (ODE)  & 31.39 & 0.9392 & 0.052 & 12.24 \\
    UniDB (SDE) & 32.05 & 0.9036 & \underline{0.045} & \underline{17.65} &
    UniDB (SDE) & 25.46 & 0.6856 & \underline{0.179} & \underline{16.21} &  UniDB (SDE)  & 29.20 & 0.9077 & \underline{0.036} & \underline{4.08} \\
    UniDB (ODE)  &\underline{34.68} & \underline{0.9426} & 0.074 & 31.16 &
    UniDB (ODE) & \underline{28.64} & \underline{0.8072} & 0.323 & 22.32 &  UniDB (ODE) & \underline{31.67} & \underline{0.9395} & 0.052 & 11.98 \\   
    \hline 
    SDDBMs (SDE)  &{32.94} & {0.9183} & \textbf{0.038} & \textbf{14.21} & SDDBMs (SDE) & 27.22 & 0.7522 & \textbf{0.136} & \textbf{13.70} & SDDBMs (SDE) & {30.30} & {0.9201} & \textbf{0.032} & \textbf{3.48} \\ 
     SDDBMs (ODE)  &\textbf{35.06} & \textbf{0.9464} & 0.066 & 25.75 & SDDBMs (ODE) & \textbf{28.90} & \textbf{0.8164} & 0.309 & 19.95 & SDDBMs (ODE) & \textbf{32.22} & \textbf{0.9399} & 0.051 & 11.87 \\
     \hline
  \end{tabular}
  }
  \vskip -0.1in
  \label{tab:main_results}
\end{table*}

\begin{figure}[ht]
    \centering
    \includegraphics[width=1.0\textwidth]{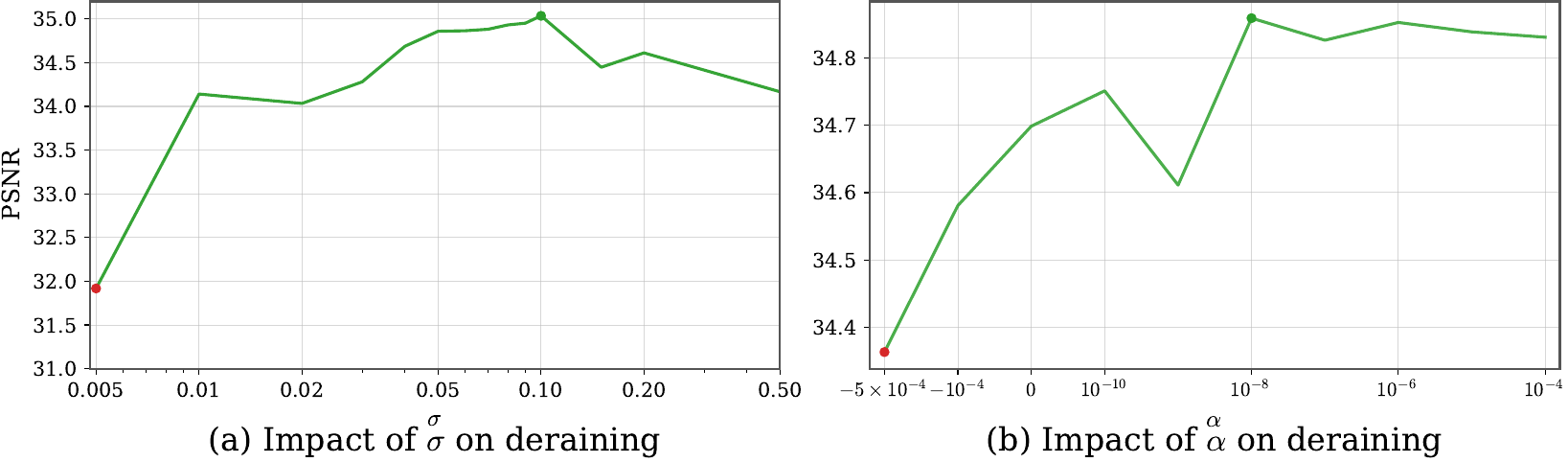}
    \caption{Performance sensitivity to hyperparameters $\sigma$ and $\alpha$ on the image deraining task.}
    \label{fig:param_alpha_sigma}
\end{figure}


\begin{figure}[ht]
    \centering
    \includegraphics[width=1.0\textwidth]{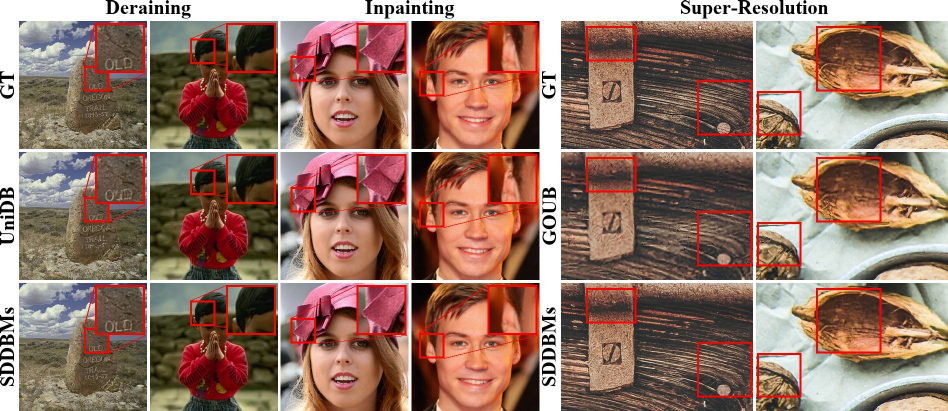}
    \caption{Qualitative comparison between SDDBMs and baseline methods on various datasets.
    }
    \label{fig:visual_dr_inp_sr}
\end{figure}

\textbf{Main Results.} We evaluate SDDBMs on image super-resolution, deraining, and inpainting using DIV2K~\citep{agustsson2017ntire}, Rain100H~\citep{yang2017deep}, and CelebA-HQ ($256 \times 256$)~\citep{karras2018progressive}, respectively. For super-resolution, the LR inputs are bicubically upsampled to match the spatial dimensions of their corresponding HR targets. For deraining, following prior works~\citep{luo2023image,yue2023image, zhu2025unidb}, PSNR and SSIM are evaluated on the Y channel (YCbCr space). For inpainting, we adopt the 100 thin masks provided by IR-SDE~\citep{luo2023image} for testing. As reported in Table~\ref{tab:main_results}, SDDBMs consistently outperform the competing methods across all three tasks and evaluation metrics. The qualitative comparisons in Figure~\ref{fig:visual_dr_inp_sr} further demonstrate that SDDBMs recover finer structural details and achieve higher visual fidelity, with additional results provided in Figures~\ref{fig:visual_inpainting} and~\ref{fig:visual_dr_sr}.

\begin{wrapfigure}{r}{0.5\textwidth}
\vspace*{0pt}
\centering
\includegraphics[width=0.5\columnwidth]{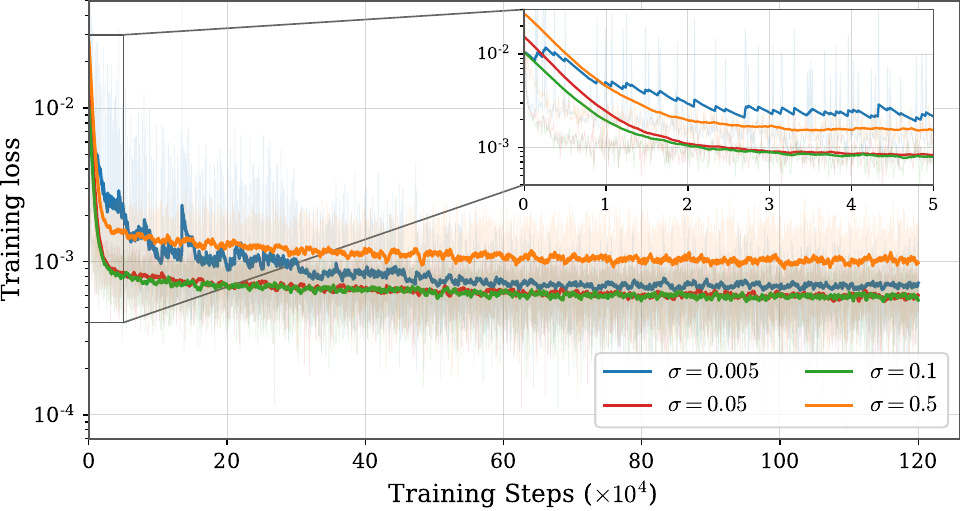}
\caption{Training loss for different values of $\sigma$ on the deraining task.}
\label{fig:sigma_deraining}
\vspace*{-20pt}
\end{wrapfigure}

We also analyze the sensitivity of our framework to key hyperparameters: $\alpha$ and $\sigma$. 
For a detailed discussion of the motivations behind these parameter choices, together with the remaining experimental results, please refer to Appendix~\ref{sec-ablation-study}.

We first study the effect of the terminal variance parameter $\sigma^2$, which controls the strength of the prescribed Gaussian terminal constraint in SDDBMs. As shown in Figure~\ref{fig:param_alpha_sigma}(a), the performance first improves as $\sigma$ increases but degrades when $\sigma$ becomes too large. An overly small $\sigma$ makes the terminal distribution highly concentrated, causing the process to approach a hard endpoint bridge and become more susceptible to terminal sampling instability, whereas an excessively large $\sigma$ over-relaxes the terminal constraint and weakens target-guided transport. Figure~\ref{fig:sigma_deraining} further supports this observation: both extremely small and large $\sigma$ lead to higher training losses, while moderate values stabilize training. Overall, $\sigma \in [0.05,0.11]$ provides a stable range,
with the best performance typically achieved for $\sigma \in [0.08,0.11]$. Similar trends are also observed on the other two tasks, as shown in Figure~\ref{fig:param_sigma_loss}(a)(b).

We further investigate the effect of $\alpha$. As shown in Figure~\ref{fig:param_alpha_sigma}(b), 
when $\alpha$ approaches the theoretical boundary $-e^{-\bar{\theta}_{0:T}}\hat{\sigma}^2 / \bar{\sigma}_{0:T}^2$, the model degrades markedly and becomes unstable.
This performance collapse empirically validates the singularity risk proven in Lemma~\ref{lem-terminal-relaxion}. Conversely, keeping $\alpha$ within an appropriate range ensures stable sampling trajectories and optimal restoration quality.



\section{Conclusion}
\label{sec:conclusion}
In this paper, we introduce Soft Denoising Diffusion Bridge Models (SDDBMs) to fundamentally resolve the terminal singularities inherent in hard-constrained diffusion bridges. Driven by the theoretical insight that exact endpoint collapse inevitably causes drift blow-up, our framework prescribes a non-degenerate Gaussian terminal marginal to establish soft-constrained bridge dynamics. This top-down formulation yields a fully tractable system characterized by a closed-form $h$-function and exact forward marginals. Building upon this probabilistic foundation, SDDBMs further unify DDBMs, GOUB, and UniDB as special cases while achieving state-of-the-art performance across diverse image-to-image translation and restoration tasks.


\subsection*{AI use statement}

We used ChatGPT to assist with manuscript revision and technical review,
including language editing and translation, improving the organization
and exposition of author-written mathematical derivations, and checking
proofs, references, and author-written code for potential errors.

Beyond these editorial and review activities, generative AI was not
involved in the theoretical development and proof construction or in
the experimental workflow, from methodological design and implementation
to data preparation, analysis, and interpretation.

All AI-assisted content was reviewed by the authors, who take full
responsibility for the accuracy and integrity of the final manuscript,
including its mathematical claims, implementation, and experimental
results.




\bibliography{iclr2027_conference}
\bibliographystyle{iclr2027_conference}
\newpage
\appendix
\renewcommand{\thesection}{A.} 
\renewcommand{\thetheorem}{\arabic{theorem}}
\renewcommand{\theequation}{\arabic{equation}}
\renewcommand{\thesubsection}{A.\arabic{subsection}}

\begin{center}
    {\Large \bf Appendix: Soft Denoising Diffusion Bridge Models}
\end{center}

\renewcommand{\thesection}{\Alph{section}}
\renewcommand{\thesubsection}{\thesection.\arabic{subsection}}

\makeatletter
\renewcommand{\@seccntformat}[1]{\csname the#1\endcsname.\quad}
\makeatother

\section{Related works}
\label{sec:related_works}
\textbf{Diffusion models.} Since the inception of diffusion models \citep{sohl2015deep,song2019generative,ho2020denoising,song2021scorebased}, various refinements have been proposed to enhance their performance. Research in this field primarily spans network design ~\citep{song2021scorebased,ho2020denoising,karras2022elucidating,peebles2023scalable}, various noise schedules \citep{nichol2021improved,karras2022elucidating}, and sampling acceleration. The latter includes both training-free numerical solvers ~\citep{lu2022dpm,song2021ddim,zhang2023deis} and training-based methods ~\citep{song2023consistency,frans2025shortcut,geng2025meanflows}. Despite these successes, standard diffusion models typically rely on Gaussian priors, which limits their flexibility in tasks requiring specific endpoint-to-endpoint mapping.

\textbf{Conditional generation.} Approaches for conditional synthesis generally follow two primary trajectories. The first category focuses on modeling the conditional probability $p(\mathbf{x}|\mathbf{y})$, a paradigm widely adopted in scenarios such as text-to-image generation~\citep{dhariwal2021diffusion,ho2022classifier,rombach2022high}. The second category directly models a transport between two distinct distributions. For instance, 
inspired by the continuity equation, flow matching~\citep{lipman2023flow, albergo2025stochastic,liu2023flow} learns a transport map between two distributions.
Schr{\"o}dinger bridges ~\citep{shi2023diffusion,de2021diffusion} seek an optimal stochastic process $\pi^*$ for inter-distribution transport under KL divergence minimization. 
However, the computational cost of the iterative procedures limits their practical applicability.
Alternatively, diffusion bridge models establish deterministic endpoint processes by combining linear SDEs with Doob's $h$-transform. While successful in image-to-image translation, these models~\citep{yue2023image,zhou2024denoising,li2023bbdm} encounter numerical singularities at the terminal boundary  which prevent direct sampling from the terminal time---a critical limitation that our work aims to resolve.

\textbf{Diffusion models with Stochastic Optimal Control.} 
In recent years, many studies have sought to integrate SOC with diffusion models.
DIS~\citep{berner2024optimal} established a connection between stochastic optimal control and diffusion models, 
 potentially enabling the transfer of methods from optimal control theory to generative modeling. Rb-modulation~\citep{rout2025rbmodulation} built on concepts from stochastic optimal control to modulate the drift field of reverse diffusion dynamics for training-free style transfer. UniDB~\citep{zhu2025unidb} offered a new perspective on diffusion bridges based on SOC which is also a special case of our SDDBMs.

\section{Complete Setting and Detailed Justifications for Affine Diffusion Dynamics}
\label{sec:full_assumptions_and_justifications}
We now ask what kind of terminal distribution can be generated by a
finite-time bridge dynamics if its drift remains numerically regular up to
the terminal time. To make this question concrete, we focus on affine
diffusion dynamics of the form: 
\[
\mathrm{d}\mathbf x_t=(A_t\mathbf x_t+\mathbf b_t)\mathrm{d}t+g_t \mathrm{d}\mathbf W_t.
\]
This class is particularly relevant for diffusion bridge models, because
many widely used bridge constructions have an affine analytic bridge drift,
or an affine \(h\)-drift correction, at the level of their closed-form
bridge term. Examples include Brownian bridges, GOU/GOUB-type bridges,
UniDB-GOU-type finite-penalty bridges, and many DDBMs bridge kernels.
Although the full reverse-time sampler may additionally contain a learned
nonlinear score network, the explicit terminal bridge correction often
takes the affine form
$\mathbf B_h(\mathbf x,t)=A_t\mathbf x+\mathbf b_t$.
Thus, affine dynamics provide a useful setting in which the source of
terminal instability can be isolated through the time-dependent
coefficients \(A_t\) and \(\mathbf b_t\).

We further consider the common setting where \(A_t\), \(\mathbf b_t\), and
\(g_t\) are independent of the source state \(\mathbf x_0\). Such an
\(x_0\)-free forward formulation is often desirable when deriving the
corresponding reverse-time dynamics, since the forward process can then be
specified without explicit access to the source sample. 
We also assume that
{$0<\int_0^T g_t^2\,\mathrm{d}t<\infty$}.
This condition ensures that the dynamics injects a nonzero but finite
amount of stochastic noise: the terminal distribution is therefore
non-deterministic, while the total diffusion energy remains finite.
Finally, to exclude drift-level
singularity, we require the affine coefficients to remain finite up to the
terminal boundary:
\[
\sup_{0\le t\le T}\|A_t\|<\infty,
\qquad
\sup_{0\le t\le T}\|\mathbf b_t\|<\infty.
\]
This condition excludes coefficient-level singularities anywhere on the
finite interval \([0,T]\). 
This requirement is directly tied to neural
parameterization and numerical stability. When these coefficients are to
be represented or approximated by a neural network, a finite-parameter
network with finite weights can only produce finite coefficient values on
the compact time interval \([0,T]\). If the target coefficients become
unbounded at any time, the network cannot represent them uniformly without
leaving the finite-valued regime, and the resulting dynamics may suffer
from unstable training, exploding gradients, or amplified sampling errors.
The terminal-time blow-up in hard endpoint bridges provides a canonical
example of this coefficient-level singularity.

\section{Singularities of Exact Endpoint Conditioning}
\label{sec-measure-singularity}
As discussed in Section~\ref{sec:doob_h_transform}, conventional diffusion
bridge models~\citep{yue2023image} construct a hard endpoint bridge by
choosing the \(h\)-function as the transition density to a prescribed
terminal point \(\mathbf x^\star\), given by 
$h_{\mathrm{hard}}(\mathbf x_t,t)
=
\int
\delta(\mathbf z-\mathbf x^\star)
p(\mathbf z,T\mid \mathbf x_t,t)\,\mathrm d\mathbf z$.
This enforces exact terminal conditioning, requiring the process to hit
\(\mathbf x^\star\) at time \(T\)~\citep{sarkka2019applied,heng2025simulating},
which causes the terminal law to degenerate to a Dirac measure 
$p_{h_{\mathrm{hard}}}(\mathbf{x}_T,T\mid\mathbf{x}_0,0)
=
\delta(\mathbf{x}_T - \mathbf{x}^\star)$.

Proposition~\ref{thm:affine_nonsingular_terminal} shows that a finite-coefficient
affine bridge dynamics yields a non-degenerate Gaussian terminal law whose
density is smooth, continuous, bounded, and still depends on the source
state \(\mathbf x_0\). Exact endpoint conditioning violates both properties,
leading to two closely related singularities.

\textbf{Singularity of exact noise cancellation.}
A defining characteristic of diffusion processes is the continuous injection of stochastic noise. Under the reference dynamics, this perpetual randomness naturally yields a terminal distribution with strictly positive variance. In stark contrast, exact endpoint conditioning requires the process to reach a precise, deterministic state $\mathbf{x}^\star$, forcing the terminal variance to collapse exactly to zero. Because the diffusion term continuously injects noise up to the final moment $T$, the only way to perfectly counteract this inherent randomness is to apply an overwhelmingly strong deterministic correction. Consequently, to completely suppress the local stochasticity and force a Dirac terminal law, the effective drift must inevitably blow up to infinity as $t \to T$.

\textbf{Finite-time source-forgetting singularity.}
The hard constraint also imposes exact source forgetting: regardless of
\(\mathbf x_0\), all admissible paths are forced to terminate at the same
endpoint \(\mathbf x^\star\). In contrast,
Proposition~\ref{thm:affine_nonsingular_terminal} shows that, within
finite-time \(\mathbf x_0\)-free affine bridge dynamics, finite drift
coefficients necessarily preserve source dependence through the terminal
mean. Indeed, the only source-dependent term in the terminal mean is
\(\Phi(T,0)\mathbf x_0\). Exact source forgetting would require the
source-propagation map \(\Phi(T,0)\) to degenerate to the zero map, whereas
bounded affine coefficients make \(\Phi(T,0)\) invertible. Thus, within
this affine \(\mathbf x_0\)-free setting, hard endpoint collapse cannot
occur in the finite-coefficient regime and can only arise through a
terminal-time blow-up of the drift coefficients.

These singularities manifest as terminal drift blow-up. As \(t\to T\), both the conditional variance \(\bar\sigma_{t:T}^2\) of the reference GOU process and the time-$t$ variance of the corresponding hard endpoint bridge vanish, with the bridge marginal converging to a Dirac measure. This produces sharp or unbounded score targets,
amplifies discretization errors, and places a severe burden on
finite-capacity neural networks and finite-step samplers.


\section{Explicit Parametrization}
\label{sec:explicit_param}
By Eqs.~(\ref{eq:new_sde}) and~(\ref{eq-xhat-xzero}), we have 
\begin{equation}\label{eq-sde-xzeorxstarmu}
\scalebox{0.93}{$\displaystyle
\begin{aligned}
   \mathrm{d}\mathbf{x}_t = & \Bigg\{
    - \left(\theta_t + g^2_t \frac{e^{-2\bar{\theta}_{t:T}}}{\hat{\sigma}^2+\bar{\sigma}^2_{t:T}}\right) \mathbf{x}_t 
    + \left[\theta_t + g^2_t\frac{e^{-\bar{\theta}_{t:T}}}{\hat{\sigma}^2+ \bar{\sigma}^2_{t:T}}\left(e^{-\bar{\theta}_{t:T}}- \frac{
\sigma^2 \left(1-e^{-\bar{\theta}_{0:T}}\right)
}
{
\bar{\sigma}_{0:T}^2-\sigma^2
}-1\right)\right]\boldsymbol{\mu} \\
    &+ g^2_t \frac{e^{-\bar{\theta}_{t:T}}}{\hat{\sigma}^2+\bar{\sigma}^2_{t:T}} \left(- \frac{
\sigma^2 e^{-\bar{\theta}_{0:T}}
}
{
\bar{\sigma}_{0:T}^2-\sigma^2
} \mathbf{x}_0 + \frac{
\bar{\sigma}_{0:T}^2
}
{
\bar{\sigma}_{0:T}^2-\sigma^2
} \mathbf{x}^\star\right)
    \Bigg\} \mathrm{d} t + g_t \mathrm{d}\mathbf{w}_t.
\end{aligned}
$}
\end{equation}
Eq.~(\ref{eq-sde-xzeorxstarmu}) is an explicit parametrization of the
soft \(h\)-transformed SDE in Eq.~(\ref{eq:new_sde}). 
It gives a drift that is affine in
\(\mathbf{x}_t\) and depends explicitly only on the source state
\(\mathbf{x}_0\), the target state \(\mathbf{x}^{\star}\), and the prior
mean \(\boldsymbol{\mu}\).
Its parameters are chosen so that the terminal marginal under the
\(h\)-transformed path measure is exactly
$p_h(\mathbf{x}_T,T\mid\mathbf{x}_0,0)
=
\mathcal N
\left(
\mathbf{x}_T;
\mathbf{x}^{\star},
\sigma^2\mathbf I
\right)$.

\section{\texorpdfstring{\(\mathbf{x}_0\)-Free Reformulation}{x0-Free Reformulation}}
\label{sec-remove-xzero}
Recall that $\hat{\mathbf{x}} = \alpha \mathbf{x}_0 + \beta \mathbf{x}^\star + \gamma \boldsymbol{\mu}$, where $\alpha, \beta$, and $\gamma$ are constants independent of $t$. 
By Eq. (\ref{eq:gou_transition_variance_mean}), we have $\hat{\mathbf{x}} - \bar{\boldsymbol{\mu}}_{t:T} = \alpha \mathbf{x}_0 + \beta \mathbf{x}^\star -  e^{-\bar{\theta}_{t:T}}\mathbf{x}_t
    +
    \left(e^{-\bar{\theta}_{t:T}}+\gamma - 1
    \right)\boldsymbol{\mu}$.
Consequently, Eq. (\ref{eq:new_sde}) becomes
\begin{equation}
    \begin{aligned}
       \mathrm{d}\mathbf{x}_t=& \Bigg\{- \left(\theta_t + g^2_t \frac{e^{-2\bar{\theta}_{t:T}}}{\hat{\sigma}^2+\bar{\sigma}^2_{t:T}}\right) \mathbf{x}_t + \left[\theta_t + g^2_t\frac{e^{-\bar{\theta}_{t:T}}}{\hat{\sigma}^2+ \bar{\sigma}^2_{t:T}}\left(e^{-\bar{\theta}_{t:T}}+\gamma-1\right)\right]\boldsymbol{\mu} \\
        &+ g^2_t \frac{e^{-\bar{\theta}_{t:T}}}{\hat{\sigma}^2+\bar{\sigma}^2_{t:T}} \left(\alpha \mathbf{x}_0 + \beta \mathbf{x}^\star\right)\Bigg\} \mathrm{d} t + g_t \mathrm{d}\mathbf{w}_t.
    \end{aligned}
    \label{eq:x_hat_sde}
\end{equation}

The purpose of this section is to derive an \(\mathbf{x}_0\)-free
reformulation of the soft \(h\)-transformed dynamics in
Eq.~(\ref{eq:x_hat_sde}). Although Eq.~(\ref{eq:x_hat_sde}) is
well-defined and nonsingular, its drift explicitly depends on the initial
state \(\mathbf{x}_0\). For deriving the corresponding reverse-time
dynamics, it is often desirable to express the forward process only in
terms of the current state \(\mathbf{x}_t\), the target state
\(\mathbf{x}^{\star}\), and the prior mean \(\boldsymbol{\mu}\). To this
end, we construct an equivalent linear SDE whose drift no longer contains
\(\mathbf{x}_0\) explicitly.

We perform this reformulation in two steps. First, we analyze the
deterministic flow induced by the drift of Eq.~(\ref{eq:x_hat_sde}). The
closed-form flow allows us to express \(\mathbf{x}_0\) as a function of
\((\mathbf{x}_t,\mathbf{x}^{\star},\boldsymbol{\mu})\), which determines
the drift coefficients of the \(\mathbf{x}_0\)-free SDE. Second, since
modifying the drift also changes the variance evolution of the process, we
calibrate the diffusion coefficient \(\eta_t\) so that the time marginals
of the reformulated SDE remain consistent with the \(h\)-transformed
marginal distribution in Proposition~\ref{prop:ph}. 

As shown in Proposition~\ref{thm:affine_nonsingular_terminal}, any finite-time affine bridge with non-singular, initial-state-independent coefficients inherently retains a dependence on the source state $\mathbf{x}_0$.
We therefore introduce the
\(\rho\)-relaxed terminal matching
$\mathbf{x}^{\ast}
=
\mathbf{x}^{\star}
+
\rho(\mathbf{x}_0-\mathbf{x}^{\star}),
\,
\rho>0 $.
Instead of forcing the terminal Gaussian center to be exactly
\(\mathbf{x}^{\star}\) and fully independent of \(\mathbf{x}_0\), this
relaxation allows a controlled source dependence. For small \(\rho\),
\(\mathbf{x}^{\ast}\) remains close to \(\mathbf{x}^{\star}\), so the bridge
is still guided toward the target while avoiding the finite-time
source-forgetting singularity caused by exact terminal matching.

\subsection{\texorpdfstring{$\mathbf{x}_0$-Free Drift Reformulation}{x0-Free Drift Reformulation}}
\label{sec:x_0}
We first consider the deterministic dynamics obtained by removing the
stochastic term from Eq.~(\ref{eq:x_hat_sde}). The resulting trajectory
admits the following closed-form solution.
\begin{restatable}{proposition}{newsdeodesolution}
\label{prop:new_sde_ode_solution}
The deterministic trajectory obtained by removing the Brownian noise term from Eq.~(\ref{eq:x_hat_sde}) admits the following closed-form solution:
\begin{equation}
    \begin{aligned}
        \mathbf{x}_t
        =&
        \frac{
        e^{-\bar{\theta}_{0:t}}
        \left(\hat{\sigma}^2+\bar{\sigma}^2_{t:T}\right)
        +\alpha e^{\bar{\theta}_{t:T}}
        \left(\bar{\sigma}^2_{0:T}-\bar{\sigma}^2_{t:T}\right)
        }
        {
        \hat{\sigma}^2+\bar{\sigma}^2_{0:T}
        }
        \mathbf{x}_0
        +
        \frac{
        \beta e^{\bar{\theta}_{t:T}}
        \left(\bar{\sigma}^2_{0:T}-\bar{\sigma}^2_{t:T}\right)
        }
        {
        \hat{\sigma}^2+\bar{\sigma}^2_{0:T}
        }
        \mathbf{x}^{\star}
        \\
        &+
        \left[
        1-
        \frac{
        e^{-\bar{\theta}_{0:t}}
        \left(\hat{\sigma}^2+\bar{\sigma}^2_{t:T}\right)
        +(1-\gamma)e^{\bar{\theta}_{t:T}}
        \left(\bar{\sigma}^2_{0:T}-\bar{\sigma}^2_{t:T}\right)
        }
        {
        \hat{\sigma}^2+\bar{\sigma}^2_{0:T}
        }
        \right]\boldsymbol{\mu}.
    \end{aligned}
    \label{eq:new_sde_ode_solution}
\end{equation}
\end{restatable}

\begin{proof}
Consider the stochastic process \(\mathbf{x}_t\) governed by
Eq.~(\ref{eq:x_hat_sde}). 
Define $\mathbf{m}_t := \mathbb{E}[\mathbf{x}_t]$.
Since Eq.~(\ref{eq:x_hat_sde}) is linear in
\(\mathbf{x}_t\) and has an additive diffusion coefficient independent of
\(\mathbf{x}_t\), 
the deterministic trajectory obtained by removing the Brownian noise term from Eq.~(\ref{eq:x_hat_sde}) is precisely \(\mathbf{m}_t\).
Moreover,
Eq.~(\ref{eq:ph}) implies
\begin{align}
\mathbf{x}_t \sim \mathcal N
\left(
\bar{\boldsymbol{\mu}}_{0:t}
+
\frac{
e^{-\bar{\theta}_{t:T}}\bar{\sigma}_{0:t}^2
}
{
\hat{\sigma}^2+\bar{\sigma}_{0:T}^2
}
(\hat{\mathbf{x}}-\bar{\boldsymbol{\mu}}_{0:T}),
\frac{
\bar{\sigma}_{0:t}^2
(\hat{\sigma}^2+\bar{\sigma}_{t:T}^2)
}
{
\hat{\sigma}^2+\bar{\sigma}_{0:T}^2
}
\mathbf I
\right).
\end{align}
Hence, 
\begin{equation}
    \mathbf{m}_t
    =
    \bar{\boldsymbol{\mu}}_{0:t}
    +
    \frac{
    e^{-\bar{\theta}_{t:T}}\bar{\sigma}_{0:t}^2
    }
    {
    \hat{\sigma}^2+\bar{\sigma}_{0:T}^2
    }
    \left(
    \hat{\mathbf{x}}
    -
    \bar{\boldsymbol{\mu}}_{0:T}
    \right).
    \label{eq:mean_h_process}
\end{equation}
Recall that
\begin{equation}
    \hat{\mathbf{x}}
    =
    \alpha\mathbf{x}_0
    +
    \beta\mathbf{x}^{\star}
    +
    \gamma\boldsymbol{\mu}.
    \label{eq:xhat_alpha_beta_gamma}
\end{equation}
Moreover, for the GOU process, 
\begin{equation}  \label{eq:gou_means}
    \bar{\boldsymbol{\mu}}_{0:T}
    =
    e^{-\bar{\theta}_{0:T}}\mathbf{x}_0
    +
    \left(
    1-e^{-\bar{\theta}_{0:T}}
    \right)\boldsymbol{\mu}, \qquad
    \bar{\sigma}_{0:T}^2
    =
    e^{-2\bar{\theta}_{t:T}}
    \bar{\sigma}_{0:t}^2
    +
    \bar{\sigma}_{t:T}^2.
\end{equation}
For compactness, define
\begin{equation}
    K_t:=
    \frac{
    e^{\bar{\theta}_{t:T}}
    \left(
    \bar{\sigma}_{0:T}^2-\bar{\sigma}_{t:T}^2
    \right)
    }
    {\hat{\sigma}^2+\bar{\sigma}_{0:T}^2}.
    \label{eq:Kt_def}
\end{equation}
Substituting Eqs.~(\ref{eq:xhat_alpha_beta_gamma}) and~(\ref{eq:gou_means})
into Eq.~(\ref{eq:mean_h_process})
gives
\begin{equation}
    \begin{aligned}
    \mathbf{m}_t
    =&
    e^{-\bar{\theta}_{0:t}}\mathbf{x}_0
    +
    \left(
    1-e^{-\bar{\theta}_{0:t}}
    \right)\boldsymbol{\mu}
    \\
    &+
    K_t
    \Big[
    \alpha\mathbf{x}_0
    +
    \beta\mathbf{x}^{\star}
    +
    \gamma\boldsymbol{\mu}
    -
    e^{-\bar{\theta}_{0:T}}\mathbf{x}_0
    -
    \left(
    1-e^{-\bar{\theta}_{0:T}}
    \right)\boldsymbol{\mu}
    \Big].
    \end{aligned}
    \label{eq:mean_expand_before_collecting}
\end{equation}
Hence,
\begin{equation}
    \begin{aligned}
    \mathbf{m}_t
    =&
    \left[
    e^{-\bar{\theta}_{0:t}}
    +
    K_t
    \left(
    \alpha-e^{-\bar{\theta}_{0:T}}
    \right)
    \right]\mathbf{x}_0
    +
    \beta K_t\mathbf{x}^{\star}
    \\
    &+
    \left[
    1-e^{-\bar{\theta}_{0:t}}
    +
    K_t
    \left(
    \gamma-1+e^{-\bar{\theta}_{0:T}}
    \right)
    \right]\boldsymbol{\mu}.
    \end{aligned}
    \label{eq:mean_expand_collected}
\end{equation}
Using $e^{-\bar{\theta}_{0:T}}
=
e^{-\bar{\theta}_{0:t}}e^{-\bar{\theta}_{t:T}}$,
we have
\begin{equation}
    K_t e^{-\bar{\theta}_{0:T}}
    =
    e^{-\bar{\theta}_{0:t}}
    \frac{
    \bar{\sigma}_{0:T}^2-\bar{\sigma}_{t:T}^2
    }
    {
    \hat{\sigma}^2+\bar{\sigma}_{0:T}^2
    }.
\end{equation}
Therefore,
\begin{equation}
    e^{-\bar{\theta}_{0:t}}
    -
    K_t e^{-\bar{\theta}_{0:T}}
    =
    e^{-\bar{\theta}_{0:t}}
    \frac{
    \hat{\sigma}^2+\bar{\sigma}_{t:T}^2
    }
    {
    \hat{\sigma}^2+\bar{\sigma}_{0:T}^2
    }.
    \label{eq:key_coeff_identity}
\end{equation}
Substituting Eq.~(\ref{eq:key_coeff_identity}) and the definition of \(K_t\)
into Eq.~(\ref{eq:mean_expand_collected}) yields
\begin{equation}
    \begin{aligned}
        \mathbf{m}_t
        =&
        \frac{
        e^{-\bar{\theta}_{0:t}}
        \left(
        \hat{\sigma}^2+\bar{\sigma}_{t:T}^2
        \right)
        +
        \alpha e^{\bar{\theta}_{t:T}}
        \left(
        \bar{\sigma}_{0:T}^2-\bar{\sigma}_{t:T}^2
        \right)
        }
        {
        \hat{\sigma}^2+\bar{\sigma}_{0:T}^2
        }
        \mathbf{x}_0
        +
        \frac{
        \beta e^{\bar{\theta}_{t:T}}
        \left(
        \bar{\sigma}_{0:T}^2-\bar{\sigma}_{t:T}^2
        \right)
        }
        {
        \hat{\sigma}^2+\bar{\sigma}_{0:T}^2
        }
        \mathbf{x}^{\star}
        \\
        &+
        \left[
        1-
        \frac{
        e^{-\bar{\theta}_{0:t}}
        \left(
        \hat{\sigma}^2+\bar{\sigma}_{t:T}^2
        \right)
        +
        (1-\gamma)e^{\bar{\theta}_{t:T}}
        \left(
        \bar{\sigma}_{0:T}^2-\bar{\sigma}_{t:T}^2
        \right)
        }
        {
        \hat{\sigma}^2+\bar{\sigma}_{0:T}^2
        }
        \right]\boldsymbol{\mu}.
    \end{aligned}
    \label{eq:mean_final_expanded}
\end{equation}
The proposition is proved.
\end{proof}


Rearranging
Eq.~(\ref{eq:new_sde_ode_solution}), we obtain 
\begin{equation}
\begin{aligned}
    \mathbf{x}_0
    =&
    \frac{
    \left(\hat{\sigma}^2+\bar{\sigma}^2_{0:T}\right)\mathbf{x}_t
    -\beta e^{\bar{\theta}_{t:T}}
    \left(\bar{\sigma}^2_{0:T}-\bar{\sigma}^2_{t:T}\right)
    \mathbf{x}^{\star}
    }
    {
    \phi_t
    }
    \\
    &-
    \frac{
    \left[
    \left(\hat{\sigma}^2+\bar{\sigma}^2_{0:T}\right)
    -e^{-\bar{\theta}_{0:t}}
    \left(\hat{\sigma}^2+\bar{\sigma}^2_{t:T}\right)
    -(1-\gamma)e^{\bar{\theta}_{t:T}}
    \left(\bar{\sigma}^2_{0:T}-\bar{\sigma}^2_{t:T}\right)
    \right]\boldsymbol{\mu}
    }
    {
    \phi_t
    },
\end{aligned}
\label{eq:x_0_x_hat}
\end{equation}

where $\phi_t$ is defined in Eq.~(\ref{eq-def-phit}).

Substituting $\mathbf{x}_0$ into the drift of
Eq.~(\ref{eq:x_hat_sde}) gives the reformulated forward dynamics
\begin{equation}\label{eq:final_sde}
   \mathrm{d}\mathbf{x}_t
    =
    \left(
    f_t\mathbf{x}_t
    +
    m_t\mathbf{x}^{\star}
    +
    h_t\boldsymbol{\mu}
    \right)\mathrm{d} t
    +
    \eta_t \mathrm{d}\mathbf{w}_t,
\end{equation}
where $\eta_t$ is an undetermined diffusion coefficient, and $f_t$, $m_t$ and 
$h_t$ are defined in Eqs.~(\ref{eq:ft_def}) and~(\ref{eq:ht_def}), respectively.

\subsection{\texorpdfstring{Well-Posedness via \(\rho\)-Relaxed Terminal Matching}{Well-Posedness via rho-Relaxed Terminal Matching}}
\label{sec:rho}
Note that Eq.~(\ref{eq:final_sde}) becomes ill-defined whenever
\(\phi_t=0\), since the drift coefficients contain \(1/\phi_t\).
Therefore, to ensure that the reformulated dynamics remain well-defined
and free of such singularities, we require \(\phi_t\neq 0\) over the time
interval of interest. The following lemma gives an equivalent
characterization of this condition.
Recall that $\hat{\mathbf{x}} = \alpha \mathbf{x}_0 + \beta \mathbf{x}^\star + \gamma \boldsymbol{\mu}$
and 
$\mathbf{x}^{\ast}
=
\mathbf{x}^{\star}
+
\rho(\mathbf{x}_0-\mathbf{x}^{\star})$.


\begin{restatable}{lemma}{lemterminalrelaxion}
\label{lem-terminal-relaxion}
{The following statements are equivalent:
(i)~\(\phi_t \neq 0\) for all \(t\in[0,T]\);
(ii)~\(\alpha > -e^{-\bar{\theta}_{0:T}}\hat{\sigma}^2 / \bar{\sigma}_{0:T}^2\); and
(iii)~\(\rho>0\).}
\end{restatable}

\begin{proof}
First, we show that  
\begin{align}\label{eq-alpha-rho-relation}
    \alpha
    >
    -\frac{
    e^{-\bar{\theta}_{0:T}}\hat{\sigma}^2    
    }
    {
    \bar{\sigma}_{0:T}^2
    }\qquad \rightleftharpoons \qquad \rho>0
\end{align}
Recall that 
$\mathbf{x}^{\ast}
=
\mathbf{x}^{\star}
+
\rho(\mathbf{x}_0-\mathbf{x}^{\star})$.
Thus, by Eqs. (\ref{eq-xhat-sigmahat}) and (\ref{eq:gou_transition_variance_mean}), we have
\begin{align*}
\hat{\mathbf{x}} &= \frac{
\bar{\sigma}_{0:T}^2
}
{
\bar{\sigma}_{0:T}^2-\sigma^2
}
\mathbf{x}^\ast
-
\frac{
\sigma^2
}
{
\bar{\sigma}_{0:T}^2-\sigma^2
}
\bar{\boldsymbol{\mu}}_{0:T} 
=  - \frac{
\sigma^2 e^{-\bar{\theta}_{0:T}}
}
{
\bar{\sigma}_{0:T}^2-\sigma^2
} \mathbf{x}_0  + \frac{
\bar{\sigma}_{0:T}^2
}
{
\bar{\sigma}_{0:T}^2-\sigma^2
}
\mathbf{x}^\ast - \frac{
\sigma^2 \left(1-e^{-\bar{\theta}_{0:T}}\right)
}
{
\bar{\sigma}_{0:T}^2-\sigma^2
}\boldsymbol{\mu}\\
& = 
\left(\frac{\rho
\bar{\sigma}_{0:T}^2
}
{
\bar{\sigma}_{0:T}^2-\sigma^2
} - \frac{
\sigma^2 e^{-\bar{\theta}_{0:T}}
}
{
\bar{\sigma}_{0:T}^2-\sigma^2
}\right) \mathbf{x}_0  + \frac{\left(1- \rho\right)
\bar{\sigma}_{0:T}^2
}
{
\bar{\sigma}_{0:T}^2-\sigma^2
} 
\mathbf{x}^\star - \frac{
\sigma^2 \left(1-e^{-\bar{\theta}_{0:T}}\right)
}
{
\bar{\sigma}_{0:T}^2-\sigma^2
}\boldsymbol{\mu}.
\end{align*}
Combined with $\hat{\mathbf{x}} = \alpha \mathbf{x}_0 + \beta \mathbf{x}^\star + \gamma \boldsymbol{\mu}$, 
we have 
\begin{align*}
\alpha = \frac{\rho
\bar{\sigma}_{0:T}^2
}
{
\bar{\sigma}_{0:T}^2-\sigma^2
} - \frac{
\sigma^2 e^{-\bar{\theta}_{0:T}}
}
{
\bar{\sigma}_{0:T}^2-\sigma^2
} ={\rho \frac{\hat{\sigma}^2+\bar{\sigma}^2_{0:T} }{\bar{\sigma}^2_{0:T}}-\frac{e^{-\bar{\theta}_{0:T}\hat{\sigma}^2}}{\bar{\sigma}^2_{0:T}}}
\end{align*}

Since \(0<\sigma^2<\bar{\sigma}_{0:T}^2\), Eq.~(\ref{eq-alpha-rho-relation}) is immediate.

Next, we show that 
\begin{align}\label{eq-alpha-phi-relation}
    \alpha
    >
    -\frac{
    e^{-\bar{\theta}_{0:T}}\hat{\sigma}^2    
    }
    {
    \bar{\sigma}_{0:T}^2
    }\qquad \rightleftharpoons \qquad \phi_t\neq 0 \quad  \forall\, t\in [0,T].
\end{align}
Then the lemma is immediate.
Recall that 
\[\phi_t
    =
    e^{-\bar{\theta}_{0:t}}
    \left(\hat{\sigma}^2+\bar{\sigma}^2_{t:T}\right)
    +
    \alpha e^{\bar{\theta}_{t:T}}
    \left(\bar{\sigma}^2_{0:T}-\bar{\sigma}^2_{t:T}\right).\]
Define
\begin{equation}\label{eq-def-alphat}
    \alpha_t = - \frac{e^{-\bar{\theta}_{0:t}}}{e^{\bar{\theta}_{t:T}}} \frac{\hat{\sigma}^2 + \bar{\sigma}^2_{t:T}}{\bar{\sigma}^2_{0:T} - \bar{\sigma}^2_{t:T}}  = - e^{-\bar{\theta}_{0:T}} \frac{\hat{\sigma}^2 + \bar{\sigma}^2_{t:T}}{\bar{\sigma}^2_{0:T} - \bar{\sigma}^2_{t:T}}.
\end{equation}

First, we show that $\phi_t\neq 0$ for all $t\in [0,T]$ implies  
\begin{align}\label{eq-alpha-lowerbound}
\alpha
    >
    \frac{-e^{-\bar{\theta}_{0:T}}\hat{\sigma}^2}{\bar{\sigma}_{0:T}^2}.
\end{align}
Assume for contradiction that $\phi_t\neq 0$ for all $t\in [0,T]$ and 
\[\alpha
    \leq
    \frac{-e^{-\bar{\theta}_{0:T}}\hat{\sigma}^2}{\bar{\sigma}_{0:T}^2}.\] 
Since $e^{\bar{\theta}_{0:T}} \geq 0$ and $\bar{\sigma}^2_{0:T}-\bar{\sigma}^2_{T:T}>0$,
we have $\phi_T \leq 0$.
Note that
\[
\phi_0 = e^{-\bar{\theta}_{0:0}}\left(\hat{\sigma}^2+\bar{\sigma}^2_{0:T}\right) = \hat{\sigma}^2+\bar{\sigma}^2_{0:T} > 0, \,\,\, \phi_T=e^{-\bar{\theta}_{0:T}}\hat{\sigma}^2+ \alpha \bar{\sigma}^2_{0:T} \leq 0. 
\]
Moreover, by Eq.~(\ref{eq:gou_transition_variance_mean}),
\(\bar{\theta}_{0:t}\), \(\bar{\theta}_{t:T}\), and
\(\bar{\sigma}_{t:T}^2\) are continuous functions of \(t\).
Therefore, \(\phi_t\) is also continuous in \(t\). Since
\(\phi_0>0\) and \(\phi_T\leq 0\), the Intermediate Value Theorem implies
that there exists some \(t^\ast\in(0,T]\) such that
$\phi_{t^\ast}=0$,
which leads to a contradiction.

Next, we show that Eq. (\ref{eq-alpha-lowerbound}) implies that $\phi_t\neq 0$ for all $t\in [0,T]$.
Assume that Eq. (\ref{eq-alpha-lowerbound}) holds.
Recall the definition of $\alpha_t$ in Eq.~(\ref{eq-def-alphat}).
For each fixed \(t\in(0,T]\), we have \(\phi_t>0\) if and only if
$\alpha>\alpha_t$.
Furthermore, since \(\bar{\sigma}_{t:T}^2\) is monotonically decreasing in
\(t\), and since
\[
e^{-\bar{\theta}_{0:T}}>0,\qquad
\hat{\sigma}^2>0,\qquad
\bar{\sigma}_{t:T}^2\ge0,
\]
it follows that the threshold \(\alpha_t\) is monotonically increasing as a
function of \(t\). 
Hence, by Eq. (\ref{eq-alpha-lowerbound}) we obtain
\begin{align}
\alpha_t \leq \alpha_T = \frac{-e^{-\bar{\theta}_{0:T}}\hat{\sigma}^2}{\bar{\sigma}_{0:T}^2} < \alpha \qquad \forall t\in (0,T].
\end{align}
Thus, $\phi_t >0$ for all $t\in (0,T]$.
Recall that  $\phi_0 >0$.
We obtain that $\phi_t >0$ for all $t\in [0,T]$.
In summary, Eq. (\ref{eq-alpha-phi-relation}) is proved.
The lemma is established.
\end{proof}

The above lemma characterizes the admissible range of \(\alpha\) under which
the reformulated dynamics remain well-defined. In particular, under the
relaxed parametrization
$\mathbf{x}^{\ast}
=
\mathbf{x}^{\star}
+
\rho(\mathbf{x}_0-\mathbf{x}^{\star})$,
the condition \(\phi_t\neq 0\) on \([0,T]\) is equivalent to
\(\rho>0\). Hence, for the \(\mathbf{x}_0\)-free reformulated dynamics, the
auxiliary terminal center \(\mathbf{x}^{\ast}\) cannot be chosen to coincide
exactly with the desired target \(\mathbf{x}^{\star}\). Instead, a small but
strictly positive relaxation is required, so that \(\mathbf{x}^{\ast}\)
remains close to \(\mathbf{x}^{\star}\) while the coefficient \(\phi_t\)
stays nonzero throughout the time interval. The $\rho$-relaxed form is used here only to illustrate that a nonsingular
$\mathbf{x}_0$-free reformulation requires a nonzero coefficient of
$\mathbf{x}_0$ in the terminal mean. SDDBMs allow more general terminal
centers subject to the admissible parameter conditions.

\subsection{{Variance Matching and Diffusion Coefficient Calibration}}

It remains to specify the diffusion coefficient $\eta_t$ in
Eq.~(\ref{eq:final_sde}). This coefficient cannot be chosen
independently of the drift. Indeed, for a linear SDE of the form
$
   \mathrm{d}\mathbf{x}_t
    =
    \left(f_t\mathbf{x}_t+\mathbf{b}_t\right)\mathrm{d} t
    +
    \eta_t \mathrm{d}\mathbf{w}_t,
$
the marginal covariance evolves according to
$
    \frac{\mathrm{d}}{\mathrm{d} t}\mathrm{Cov}(\mathbf{x}_t)
    =
    2f_t\,\mathrm{Cov}(\mathbf{x}_t)
    +
    \eta_t^2\mathbf{I}.
$
Therefore, once the drift coefficient is changed from that in
Eq.~(\ref{eq:x_hat_sde}) to the coefficient $f_t$ in
Eq.~(\ref{eq:ft_def}), the stochastic perturbation must be recalibrated
so that the reformulated process preserves the target marginal
distribution in Proposition~\ref{prop:ph}. This leads to the following
result.

\begin{restatable}{proposition}{propdiffuse}
\label{prop:diffuse}
Assume that the drift in Eq.~(\ref{eq:final_sde}) has been chosen so that
the mean trajectory matches that of the \(h\)-transformed marginal
distribution in
Proposition~\ref{prop:ph}.
Then the time marginals are consistent if the diffusion
coefficient \(\eta_t\) is chosen so that $\eta^2_t$ is given by Eq.~(\ref{eq:eta}).
\end{restatable}

\begin{proof}
Consider the process governed by
Eq.~(\ref{eq:final_sde}). Since the drift is affine in
\(\mathbf{x}_t\), the process is Gaussian at each time \(t\) when
initialized from a deterministic \(\mathbf{x}_0\). Therefore, its time
marginal is fully determined by its mean and covariance.
By construction, the drift in Eq.~(\ref{eq:final_sde}) matches the mean
trajectory of the \(h\)-transformed marginal distribution. It remains to
choose \(\eta_t\) so that the covariance also matches the target
covariance.

Eq. (\ref{eq:final_sde}) implies
\begin{equation}
   \mathrm{d}\mathbf{x}_t
    =
    \left[
    f_t\mathbf{x}_t+\mathbf{b}_t
    \right]\mathrm{d} t
    +
    \eta_t \mathrm{d}\mathbf{w}_t,
\end{equation}
where \(\mathbf{b}_t\) is deterministic and does not affect the covariance.
Thus, if
$\operatorname{Cov}(\mathbf{x}_t)=V_t\mathbf I$,
then \(V_t\) satisfies
\begin{equation}
    \frac{\mathrm{d} V_t}{\mathrm{d} t}
    =
    2f_tV_t+\eta_t^2.
    \label{eq:variance_ode}
\end{equation}


Recall from Eqs.~(\ref{eq-xhat-sigmahat}) and (\ref{eq-def-phit})  that
\begin{equation}
    \varphi_t
    =
    \sigma^2
    \left(
    \bar{\sigma}_{0:T}^2-\bar{\sigma}_{t:T}^2
    \right)
    +
    \bar{\sigma}_{0:T}^2\bar{\sigma}_{t:T}^2,\qquad 
    \hat{\sigma}^2
=
\frac{
\sigma^2\bar{\sigma}_{0:T}^2
}
{
\bar{\sigma}_{0:T}^2-\sigma^2
}.
\end{equation}
Hence,
the marginal covariance of the \(h\)-transformed marginal
distribution in
Proposition~\ref{prop:ph} is
\begin{equation}
\frac{
\bar{\sigma}_{0:t}^2
(\hat{\sigma}^2+\bar{\sigma}_{t:T}^2)
}
{
\hat{\sigma}^2+\bar{\sigma}_{0:T}^2
}
\mathbf I
    =
    \frac{
    \bar{\sigma}_{0:t}^2
    \varphi_t
    }
    {
    \bar{\sigma}_{0:T}^4
    }\mathbf I.
    \label{eq:target_variance}
\end{equation}
Define 
\begin{align}
V^{\ast}_t :=\frac{
    \bar{\sigma}_{0:t}^2
    \varphi_t
    }
    {
    \bar{\sigma}_{0:T}^4
    }\mathbf,\qquad 
    \eta_t^2
    :=
    \frac{\mathrm{d} V^{\ast}_t}{\mathrm{d} t}
    -
    2f_tV^{\ast}_t.
    \label{eq:eta_from_variance_ode}
\end{align}
We have $V^{\ast}_0 = 0$.
Moreover, we also have $V_0 = 0$ by $\operatorname{Cov}(\mathbf{x}_0)=\mathbf 0 = V_0\mathbf I$.
Hence,
Eqs. (\ref{eq:variance_ode}) and (\ref{eq:eta_from_variance_ode}) imply \(V_t=V^{\ast}_t\).
In addition, for the GOU variance, we
have
\begin{equation}
    \frac{\mathrm{d}}{\mathrm{d} t}\bar{\sigma}_{t:T}^2
    =
    -g_t^2e^{-2\bar{\theta}_{t:T}},
    \qquad
    \frac{\mathrm{d}}{\mathrm{d} t}\bar{\sigma}_{0:t}^2
    =
    g_t^2-2\theta_t\bar{\sigma}_{0:t}^2.
    \label{eq:variance_derivatives}
\end{equation}
Thus,
\begin{equation}
    \frac{\mathrm{d}\varphi_t}{\mathrm{d} t}
    =
    g_t^2e^{-2\bar{\theta}_{t:T}}
    \left(
    \sigma^2-\bar{\sigma}_{0:T}^2
    \right).
    \label{eq:varphi_derivative}
\end{equation}
Hence
\begin{equation}
    \begin{aligned}
    \frac{\mathrm{d} V^{\ast}_t}{\mathrm{d} t}
    &=
    \frac{1}{\bar{\sigma}_{0:T}^4}
    \left[
    \left(
    g_t^2-2\theta_t\bar{\sigma}_{0:t}^2
    \right)\varphi_t
    +
    \bar{\sigma}_{0:t}^2
    g_t^2e^{-2\bar{\theta}_{t:T}}
    \left(
    \sigma^2-\bar{\sigma}_{0:T}^2
    \right)
    \right].
    \end{aligned}
    \label{eq:target_variance_derivative}
\end{equation}
Substituting Eq.~(\ref{eq:target_variance_derivative}) and Eq.~(\ref{eq:ft_def}) into Eq.~(\ref{eq:eta_from_variance_ode}), we obtain
\begin{equation}
    \eta^2_t
    =
    \frac{g^2_t}{\bar{\sigma}^4_{0:T}}
    \left\{
    \varphi_t
    \left[
    1-
    \frac{
    2e^{-\bar{\theta}_{t:T}}
    \bar{\sigma}^2_{0:t}
    \left(\alpha-e^{-\bar{\theta}_{0:T}}\right)
    }
    {
    \phi_t
    }
    \right]
    +
    \bar{\sigma}^2_{0:t}
    e^{-2\bar{\theta}_{t:T}}
    \left(
    \sigma^2-\bar{\sigma}^2_{0:T}
    \right)
    \right\}.
    \label{eq:eta}
\end{equation}
Consequently, the reformulated process
has the same covariance as the target \(h\)-transformed marginal
distribution. Since the drift has already been chosen to match the target
mean and both marginals are Gaussian, the time marginals are consistent.
\end{proof}

This choice of $\eta_t$ ensures that the stochastic perturbation in
Eq.~(\ref{eq:final_sde}) compensates for the modified linear
drift and preserves the marginal covariance specified by
Proposition~\ref{prop:ph}. In this sense, Eq.~(\ref{eq:final_sde})
should be interpreted as a marginally consistent reformulation of
Eq.~(\ref{eq:x_hat_sde}), rather than as a pathwise identity obtained by
a direct substitution inside the stochastic term.

Eq.~(\ref{eq:eta}) can also be equivalently written as
\begin{equation}
    \eta_t^2
    =
    g_t^2
    \frac{
        e^{-\bar{\theta}_{0:T}}
        \left(\hat{\sigma}^2+\bar{\sigma}_{t:T}^2\right)
        -
        \alpha
        \left(\bar{\sigma}_{0:T}^2-\bar{\sigma}_{t:T}^2\right)
    }{
        e^{-\bar{\theta}_{0:T}}
        \left(\hat{\sigma}^2+\bar{\sigma}_{t:T}^2\right)
        +
        \alpha
        \left(\bar{\sigma}_{0:T}^2-\bar{\sigma}_{t:T}^2\right)
    }.
    \label{eq:eta_compact}
\end{equation}
The condition
\begin{equation}
    -\frac{
        e^{-\bar{\theta}_{0:T}}\hat{\sigma}^2
    }{
        \bar{\sigma}_{0:T}^2
    }
    < \alpha \leq
    \frac{
        e^{-\bar{\theta}_{0:T}}\hat{\sigma}^2
    }{
        \bar{\sigma}_{0:T}^2
    }
    \label{eq:alpha_admissible}
\end{equation}
ensures nonsingular drift coefficients (by Lemma~\ref{lem-terminal-relaxion}) and $\eta_t^2 \geq 0$ for all $t \in [0,T]$.

By Propositions~\ref{prop:new_sde_ode_solution} and~\ref{prop:diffuse}, the reverse SDE in Eq.~(\ref{eq:new_sde_reverse_mean_ode}) follows immediately from the frameworks established in~\citep{song2021scorebased, yue2023image}.

\section{Missing proofs and derivations}

\subsection{Proof of Proposition~\ref{thm:affine_nonsingular_terminal}}
\label{sec:thm_affine_nonsingular_terminal}

\thmterminaldistribution*

\begin{proof}
Let $R_t := \Phi(t,0)$.
Then \(R_t\) satisfies
\[
    \frac{\mathrm{d}}{\mathrm{d}t}R_t=A_tR_t,
    \qquad
    R_0=\mathbf{I}.
\]
Since \(A_t\) is uniformly bounded on the finite interval \([0,T]\), the
matrix \(R_t\) is invertible for every \(t\in[0,T]\). Moreover,
\[
    \frac{\mathrm{d}}{\mathrm{d}t}R_t^{-1}
    =
    -R_t^{-1}A_t.
\]
Define $\mathbf y_t := R_t^{-1}\mathbf x_t$.
Using Itô's product rule and the SDE $\mathrm{d}\mathbf{x}_t
    =
    (A_t\mathbf{x}_t + \mathbf{b}_t)\mathrm{d}t
    +
    g_t \mathrm{d}\mathbf{W}_t$,
\[
\begin{aligned}
    \mathrm{d}\mathbf y_t
    &=
    \mathrm{d}(R_t^{-1}\mathbf x_t) \\
    &=
    \left(\frac{\mathrm{d}}{\mathrm{d}t}R_t^{-1}\right)\mathbf x_t\mathrm{d}t
    +
    R_t^{-1}\mathrm{d}\mathbf x_t \\
    &=
    -R_t^{-1}A_t\mathbf x_t\mathrm{d}t
    +
    R_t^{-1}
    \left[
    (A_t\mathbf x_t+\mathbf b_t)\mathrm{d}t
    +
    g_t\mathrm{d}\mathbf W_t
    \right] \\
    &=
    R_t^{-1}\mathbf b_t\mathrm{d}t
    +
    R_t^{-1}g_t\mathrm{d}\mathbf W_t.
\end{aligned}
\]
Integrating from \(0\) to \(T\) gives
\[
    \mathbf y_T
    =
    \mathbf x_0
    +
    \int_0^T R_s^{-1}\mathbf b_s\,\mathrm{d}s
    +
    \int_0^T R_s^{-1}g_s\, \mathrm{d}\mathbf W_s.
\]
Multiplying both sides by \(R_T\), and using
$R_TR_s^{-1} =
    \Phi(T,s)$,
we obtain the variation-of-constants representation
\begin{equation}
    \mathbf x_T
    =
    \Phi(T,0)\mathbf x_0
    +
    \int_0^T \Phi(T,s)\mathbf b_s\, \mathrm{d}s
    +
    \int_0^T \Phi(T,s)g_s\, \mathrm{d}\mathbf W_s.
    \label{eq:affine_solution_formula}
\end{equation}
The first two terms in Eq.~(\ref{eq:affine_solution_formula}) are
deterministic given a fixed \(\mathbf x_0\). The last term is a Gaussian
random vector, since it is an Itô integral with deterministic integrand.
Therefore, $\mathbf x_T\mid \mathbf x_0$
is Gaussian with mean
\[
    \mathbf m_T(\mathbf x_0)
    =
    \Phi(T,0)\mathbf x_0
    +
    \int_0^T \Phi(T,s)\mathbf b_s\,\mathrm{d}s
\]
and covariance
\[
    \Sigma_T
    =
    \int_0^T g_s^2
    \Phi(T,s)\Phi(T,s)^\top \,\mathrm{d}s.
\]

\vspace{0.2cm}

It remains to show that \(\Sigma_T\) is positive definite. Let
\[
    M_A:=\sup_{0\le t\le T}\|A_t\|<\infty.
\]
By Gronwall's inequality, for \(0\le s\le T\),
\[
    \|\Phi(T,s)\|
    \le
    e^{M_A(T-s)}
    \le
    e^{M_AT}.
\]
Similarly, since \(\Phi(T,s)\) is invertible and
\(\Phi(T,s)^{-1}=\Phi(s,T)\), we also have
\[
    \|\Phi(T,s)^{-1}\|
    =
    \|\Phi(s,T)\|
    \le
    e^{M_A(T-s)}
    \le
    e^{M_AT}.
\]
Hence, for any \(\mathbf v\in\mathbb R^d\),
\[
    \|\Phi(T,s)^\top \mathbf v\|
    \ge
    e^{-M_AT}\|\mathbf v\|.
\]
Therefore,
\[
\begin{aligned}
    \mathbf v^\top \Sigma_T \mathbf v
    &=
    \int_0^T
    g_s^2
    \mathbf v^\top
    \Phi(T,s)\Phi(T,s)^\top
    \mathbf v
    \,\mathrm{d}s \\
    &=
    \int_0^T
    g_s^2
    \|\Phi(T,s)^\top \mathbf v\|^2\,
    \mathrm{d}s \\
    &\ge
    e^{-2M_AT}
    \left(
    \int_0^T g_s^2\,\mathrm{d}s
    \right)
    \|\mathbf v\|^2.
\end{aligned}
\]
By Eq.~(\ref{eq:bounded_affine_coefficients}), we obtain 
\[
    \mathbf v^\top \Sigma_T \mathbf v>0,\qquad \forall\, \mathbf v\neq0.
\]
Thus, $\Sigma_T\succ0$.
Since \(\Sigma_T\) is positive definite, the terminal Gaussian is
non-degenerate and has density
\begin{equation*}
    p(\mathbf x,T\mid \mathbf x_0,0)
    =
    \frac{1}{(2\pi)^{d/2}|\Sigma_T|^{1/2}}
    \exp
    \left(
    -\frac12
    (\mathbf x-\mathbf m_T(\mathbf x_0))^\top
    \Sigma_T^{-1}
    (\mathbf x-\mathbf m_T(\mathbf x_0))
    \right).
\end{equation*}
Thus, as a function of \(\mathbf x\), this density is smooth, i.e., it has continuous derivatives of all
orders. In particular, it is continuous.
Moreover, since
\[
    p(\mathbf x,T\mid \mathbf x_0,0)
    \le
    \frac{1}{(2\pi)^{d/2}|\Sigma_T|^{1/2}}  <\infty ,
\]
the density is bounded.

Finally, we prove that the terminal density depends on the source state.
For two initial states \(\mathbf x_0,\mathbf y_0\), the difference between
the terminal means is
\[
\begin{aligned}
    \mathbf m_T(\mathbf x_0)-\mathbf m_T(\mathbf y_0)
    &=
    \Phi(T,0)(\mathbf x_0-\mathbf y_0).
\end{aligned}
\]
Since \(\Phi(T,0)\) is invertible, we have
$\det \Phi(T,0)
    \neq0$.
Therefore, if \(\mathbf x_0\neq\mathbf y_0\), then
\[
    \Phi(T,0)(\mathbf x_0-\mathbf y_0)\neq0,
\]
and hence
$\mathbf m_T(\mathbf x_0)\neq \mathbf m_T(\mathbf y_0)$.
This completes the
proof.
\end{proof}
\subsection{Proof of Theorem~\ref{thm-ph-key}}
\label{sec:thm_ph-key}

\thmphkey*
\begin{proof}
By assumption,
\[
p_h(\mathbf x,T\mid\mathbf{x}_0,0)
=
\mathcal N(\mathbf x;\mathbf{x}^\ast,\sigma^2\mathbf I),
\qquad 
p(\mathbf x,T\mid\mathbf{x}_0,0)
=
\mathcal N(\mathbf x;\bar{\boldsymbol{\mu}}_{0:T},
\bar{\sigma}_{0:T}^2\mathbf I).
\]
Since \(p(\mathbf x,T\mid\mathbf{x}_0,0)>0\) for all \(\mathbf x\),
Eq.~(\ref{eq:soft_h_function_terminal_density}) implies
\begin{equation*}
    \mathcal N(\mathbf x;\mathbf{x}^\ast,\sigma^2\mathbf I)
    =
    \frac{
    \mathcal N(\mathbf x;\bar{\boldsymbol{\mu}}_{0:T},
\bar{\sigma}_{0:T}^2\mathbf I)q(\mathbf{x})
    }
    {
    h(\mathbf{x}_0,0)
    }.
\end{equation*}
Equivalently,
\[
q(\mathbf x)
=
h(\mathbf{x}_0,0)
\frac{
\mathcal N(\mathbf x;\mathbf{x}^\ast,\sigma^2\mathbf I)
}
{
\mathcal N(\mathbf x;\bar{\boldsymbol{\mu}}_{0:T},
\bar{\sigma}_{0:T}^2\mathbf I)
}.
\]
Hence,
\[
q(\mathbf x)
\propto
\exp
\left\{
-\frac12
\left[
\frac{\|\mathbf x-\mathbf{x}^\ast\|^2}{\sigma^2}
-
\frac{\|\mathbf x-\bar{\boldsymbol{\mu}}_{0:T}\|^2}
{\bar{\sigma}_{0:T}^2}
\right]
\right\}.
\]
By \[
0<\sigma^2<\bar{\sigma}_{0:T}^2,
\]
we obtain
\[
\frac1{\sigma^2}
-
\frac1{\bar{\sigma}_{0:T}^2} > 0.
\]
Hence, since $q$ is a nonnegative probability density,
we obtain
\[
q(\mathbf x)
=
\mathcal N(\mathbf x;\hat{\mathbf{x}},\hat{\sigma}^2\mathbf I),
\]
where
\begin{align*}
\hat{\mathbf{x}}
&=
\hat{\sigma}^2
\left(
\frac{\mathbf{x}^\ast}{\sigma^2}
-
\frac{\bar{\boldsymbol{\mu}}_{0:T}}{\bar{\sigma}_{0:T}^2}
\right)
=
\frac{
\bar{\sigma}_{0:T}^2
}
{
\bar{\sigma}_{0:T}^2-\sigma^2
}
\mathbf{x}^\ast
-
\frac{
\sigma^2
}
{
\bar{\sigma}_{0:T}^2-\sigma^2
}
\bar{\boldsymbol{\mu}}_{0:T},\\
\hat{\sigma}^2
&=
\left(
\frac1{\sigma^2}
-
\frac1{\bar{\sigma}_{0:T}^2}
\right)^{-1}
=
\frac{
\sigma^2\bar{\sigma}_{0:T}^2
}
{
\bar{\sigma}_{0:T}^2-\sigma^2
}.
\end{align*}
The theorem is proved.
\end{proof}

\subsection{Proof of Theorem~\ref{thm:h_function}}
\label{sec:thm_h_function}

\thmhfunction*

\begin{proof}
Using the standard product identity for Gaussian densities~\citep{petersen2012matrix},
\begin{equation}
\int_{\mathbb R^d}
\mathcal N(\mathbf z;\mathbf m_1,\Sigma_1)
\mathcal N(\mathbf z;\mathbf m_2,\Sigma_2)
\,\mathrm{d} \mathbf z
=
\mathcal N(\mathbf m_1;\mathbf m_2,\Sigma_1+\Sigma_2),
\end{equation}
we obtain
\begin{equation}
\begin{aligned}
h(\mathbf x_t,t)
&=
\int
\mathcal N(
\mathbf z;
\bar{\boldsymbol{\mu}}_{t:T},
\bar\sigma^2_{t:T}\mathbf I
)
\mathcal N(
\mathbf z;
\hat{\mathbf x},
\hat\sigma^2\mathbf I
)
\,\mathrm{d} \mathbf z \\
&=
\mathcal N(
\bar{\boldsymbol{\mu}}_{t:T};
\hat{\mathbf x},
(\bar\sigma^2_{t:T}+\hat\sigma^2)\mathbf I
) \\
&=
\frac{1}
{
[2\pi(\hat\sigma^2+\bar\sigma^2_{t:T})]^{d/2}
}
\exp\left[
-\frac{
\|\bar{\boldsymbol{\mu}}_{t:T}-\hat{\mathbf x}\|^2
}
{
2(\hat\sigma^2+\bar\sigma^2_{t:T})
}
\right].
\end{aligned}
\end{equation}

Next, we verify that $h(\mathbf{x}_t,t)$ satisfies the KBE. We first compute $\nabla_{\mathbf{x}_t} h(\mathbf{x}_t,t)$ and $\nabla^2_{\mathbf{x}_t} h(\mathbf{x}_t,t)$:
\begin{align*}
        \nabla_{\mathbf{x}_t} h(\mathbf{x}_t,t)&=-h(\mathbf{x}_t,t)\frac{1}{2(\hat{\sigma}^2+\bar{\sigma}^2_{t:T})} \frac{\partial\|\bar{\boldsymbol{\mu}}_{t:T}-\hat{\mathbf{x}}\|^2}{\partial \mathbf{x}_t} \\
        &=-h(\mathbf{x}_t,t) \left(\frac{\partial \bar{\boldsymbol{\mu}}_{t:T}}{\partial \mathbf{x}_t}\right)^\top \frac{\bar{\boldsymbol{\mu}}_{t:T}-\hat{\mathbf{x}}}{\hat{\sigma}^2+\bar{\sigma}^2_{t:T}} \\
        &=-h(\mathbf{x}_t,t) e^{-\bar{\theta}_{t:T}}\frac{\bar{\boldsymbol{\mu}}_{t:T}-\hat{\mathbf{x}}}{\hat{\sigma}^2+\bar{\sigma}_{t:T}^2},\\
        \nabla^2_{\mathbf{x}_t} h(\mathbf{x}_t,t)&=\nabla_{\mathbf{x}_t}  \left(\nabla_{\mathbf{x}_t} h(\mathbf{x}_t,t)\right)\\
        &=-\nabla_{\mathbf{x}_t} \left[h(\mathbf{x}_t,t)e^{-\bar{\theta}_{t:T}}\frac{\bar{\boldsymbol{\mu}}_{t:T}-\hat{\mathbf{x}}}{\hat{\sigma}^2+\bar{\sigma}_{t:T}^2} \right] \\
        &=-h(\mathbf{x}_t,t)\frac{e^{-2\bar{\theta}_{t:T}} \,\mathbf{I}}{\hat{\sigma}^2+\bar{\sigma}^2_{t:T}}+ h(\mathbf{x}_t,t) e^{-2\bar{\theta}_{t:T}}\frac{(\bar{\boldsymbol{\mu}}_{t:T}-\hat{\mathbf{x}}) (\bar{\boldsymbol{\mu}}_{t:T}-\hat{\mathbf{x}})^\top}{(\hat{\sigma}^2+\bar{\sigma}^2_{t:T})^2} \\
        &= h(\mathbf{x}_t,t)e^{-2\bar{\theta}_{t:T}} \left[\frac{(\bar{\boldsymbol{\mu}}_{t:T}-
        \hat{\mathbf{x}})(\bar{\boldsymbol{\mu}}_{t:T}-\hat{\mathbf{x}})^\top}{(\hat{\sigma}^2+\bar{\sigma}^2_{t:T})^2}-\frac{\mathbf{I}}{\hat{\sigma}^2+\bar{\sigma}^2_{t:T}}\, \right].
    \end{align*}
Hence, we obtain 
\begin{equation*}
    \begin{aligned}
        \Delta_{\mathbf{x}_t} h(\mathbf{x}_t,t)
        =h(\mathbf{x}_t,t)e^{-2\bar{\theta}_{t:T}} \left[\frac{\|\bar{\boldsymbol{\mu}}_{t:T}-\hat{\mathbf{x}}\|^2}{(\hat{\sigma}^2+\bar{\sigma}^2_{t:T})^2} -\frac{d}{\hat{\sigma}^2+\bar{\sigma}^2_{t:T}}\right].
    \end{aligned}
\end{equation*}   
Moreover,
\begin{equation*}
    \begin{aligned}
        \frac{\partial h(\mathbf{x}_t,t)}{\partial t}=&-\frac{d}{2}(\hat{\sigma}^2+\bar{\sigma}^2_{t:T})^{-1}\frac{\partial \bar{\sigma}^2_{t:T}}{\partial t} h(\mathbf{x}_t,t)-\frac{1}{2} h(\mathbf{x}_t,t) \frac{\partial}{\partial t} \left[\frac{\|\bar{\boldsymbol{\mu}}_{t:T}-\hat{\mathbf{x}}\|^2}{\hat{\sigma}^2+\bar{\sigma}^2_{t:T}}\right] \\
        =&h(\mathbf{x}_t,t)\frac{g^2_t e^{-2\bar{\theta}_{t:T}}d}{2(\hat{\sigma}^2+\bar{\sigma}^2_{t:T})}-\frac{1}{2}h(\mathbf{x}_t,t) \left[\frac{2(\bar{\boldsymbol{\mu}}_{t:T}-\hat{\mathbf{x}})^\top}{\hat{\sigma}^2+\bar{\sigma}^2_{t:T}}\frac{\partial \bar{\boldsymbol{\mu}}_{t:T}}{\partial t}-\frac{\|\bar{\boldsymbol{\mu}}_{t:T}-\hat{\mathbf{x}}\|^2}{(\hat{\sigma}^2+\bar{\sigma}^2_{t:T})^2} \frac{\partial \bar{\sigma}^2_{t:T}}{\partial t}\right] \\
        =&{h(\mathbf{x}_t,t)\left[\frac{g^2_t e^{-2\bar{\theta}_{t:T}}d}{2(\hat{\sigma}^2+\bar{\sigma}^2_{t:T})}-\frac{(\bar{\boldsymbol{\mu}}_{t:T}-\hat{\mathbf{x}})^\top}{\hat{\sigma}^2+\bar{\sigma}^2_{t:T}}
        \frac{\partial \bar{\boldsymbol{\mu}}_{t:T}}{\partial t}+\frac{\|\bar{\boldsymbol{\mu}}_{t:T}-\hat{\mathbf{x}}\|^2}{2(\hat{\sigma}^2+\bar{\sigma}^2_{t:T})^2}\frac{\partial \bar{\sigma}^2_{t:T}}{\partial t }\right]} \\
        =&h(\mathbf{x}_t,t)\left[\frac{g^2_t e^{-2\bar{\theta}_{t:T}}d}{2(\hat{\sigma}^2+\bar{\sigma}^2_{t:T})}-\frac{(\bar{\boldsymbol{\mu}}_{t:T}-\hat{\mathbf{x}})^\top(\mathbf{x}_t-\boldsymbol{\mu})}{\hat{\sigma}^2+\bar{\sigma}^2_{t:T}}
        e^{-\bar{\theta}_{t:T}}\theta_t-\frac{\|\bar{\boldsymbol{\mu}}_{t:T}-\hat{\mathbf{x}}\|^2}{2(\hat{\sigma}^2+\bar{\sigma}^2_{t:T})^2}g^2_t e^{-2\bar{\theta}_{t:T}} \right].
    \end{aligned}
\end{equation*}
Therefore,
\begin{equation*}
    \begin{aligned}
         &\quad \theta_t(\boldsymbol{\mu} - \mathbf{x}_t) \cdot \nabla_{\mathbf{x}_t} h(\mathbf{x}_t,t) + \frac{1}{2}g^2_t \Delta_{\mathbf{x}_t} h(\mathbf{x}_t,t)\\
        &=-\theta_t h(\mathbf{x}_t,t) \frac{(\bar{\boldsymbol{\mu}}_{t:T}-\hat{\mathbf{x}})^\top(\boldsymbol{\mu}-\mathbf{x}_t)}{\hat{\sigma}^2+\bar{\sigma}^2_{t:T}} e^{-\bar{\theta}_{t:T}}+ \frac{1}{2}g^2_t h(\mathbf{x}_t,t) e^{-2\bar{\theta}_{t:T}}\left[\frac{\|\bar{\boldsymbol{\mu}}_{t:T}-\hat{\mathbf{x}}\|^2}{(\hat{\sigma}^2+\bar{\sigma}^2_{t:T})^2} -\frac{d}{\hat{\sigma}^2+\bar{\sigma}^2_{t:T}}\right] \\
        &=-\frac{\partial h(\mathbf{x}_t,t)}{\partial t}.
    \end{aligned}
\end{equation*}   
Thus, the $h$-function satisfies the KBE.
\end{proof}

\subsection{Proof of Proposition~\ref{prop:ph}}
\label{sec:prop_ph}

\propph*

\begin{proof}
By applying the standard marginalization and conditioning properties of Gaussian distributions \citep{petersen2012matrix}, we can rewrite the product of the marginal and conditional densities. Specifically, assuming $A, S$, and $R$ are positive scalars, for isotropic Gaussian variables
\[
\mathbf{x} \sim \mathcal{N}(\mathbf{m}, S\mathbf{I}),
\qquad
\mathbf{y} \mid \mathbf{x} \sim \mathcal{N}(A\mathbf{x} + \mathbf{c}, R\mathbf{I}),
\]
the joint density $p(\mathbf{x}, \mathbf{y})$ can be factored as $p(\mathbf{y})p(\mathbf{x} \mid \mathbf{y})$, which yields the following identity:
\[
\begin{aligned}
&\mathcal{N}(\mathbf{x}; \mathbf{m}, S\mathbf{I}) \mathcal{N}(\mathbf{y}; A\mathbf{x} + \mathbf{c}, R\mathbf{I}) \\
=&
\mathcal{N}(\mathbf{y}; A\mathbf{m} + \mathbf{c}, (A^2S + R)\mathbf{I})
\mathcal{N}
\left(
\mathbf{x};
\mathbf{m} + \frac{SA}{A^2S + R} (\mathbf{y} - A\mathbf{m} - \mathbf{c}),
\frac{SR}{A^2S + R}\mathbf{I}
\right).
\end{aligned}
\]
Within our framework, we instantiate the parameters as follows:
\[
\mathbf{x} = \mathbf{x}_t, \quad
\mathbf{y} = \hat{\mathbf{x}}, \quad
\mathbf{m} = \bar{\boldsymbol{\mu}}_{0:t}, \quad
S = \bar{\sigma}_{0:t}^2,\quad 
A = e^{-\bar{\theta}_{t:T}}, \quad
\mathbf{c} = \boldsymbol{\mu}(1 - e^{-\bar{\theta}_{t:T}}), \quad
R = \hat{\sigma}^2 + \bar{\sigma}_{t:T}^2.
\]
Substituting these into the aforementioned identities yields the desired marginal parameters:
\[
A\mathbf{m} + \mathbf{c} = \bar{\boldsymbol{\mu}}_{0:T},
\]
and
\[
A^2S + R = \hat{\sigma}^2 + \bar{\sigma}_{0:T}^2.
\]

Hence, by Eqs. (\ref{eq-def-pztxtt}) and (\ref{eq:h_function_new_trans}),
\[
\begin{aligned}
p(\mathbf{x}_t,t\mid \mathbf{x}_0,0)h(\mathbf{x}_t,t)&={\mathcal{N}(\hat{\mathbf{x}};\bar{\boldsymbol{\mu}}_{0:T},\left(\hat{\sigma}^2+\bar{\sigma}^2_{0:T}\right) \mathbf{I})\mathcal N
\left(
\mathbf{x}_t;
\bar{\boldsymbol{\mu}}_{0:t}
+
\frac{
e^{-\bar{\theta}_{t:T}}\bar{\sigma}_{0:t}^2
}
{
\hat{\sigma}^2+\bar{\sigma}_{0:T}^2
}
(\hat{\mathbf{x}}-\bar{\boldsymbol{\mu}}_{0:T}),
\frac{
\bar{\sigma}_{0:t}^2
(\hat{\sigma}^2+\bar{\sigma}_{t:T}^2)
}
{
\hat{\sigma}^2+\bar{\sigma}_{0:T}^2
}
\mathbf I
\right)} \\
&=
h(\mathbf{x}_0,0)
\mathcal N
\left(
\mathbf{x}_t;
\bar{\boldsymbol{\mu}}_{0:t}
+
\frac{
e^{-\bar{\theta}_{t:T}}\bar{\sigma}_{0:t}^2
}
{
\hat{\sigma}^2+\bar{\sigma}_{0:T}^2
}
(\hat{\mathbf{x}}-\bar{\boldsymbol{\mu}}_{0:T}),
\frac{
\bar{\sigma}_{0:t}^2
(\hat{\sigma}^2+\bar{\sigma}_{t:T}^2)
}
{
\hat{\sigma}^2+\bar{\sigma}_{0:T}^2
}
\mathbf I
\right).
\end{aligned}
\]
By Eq.~(\ref{eq:general_h_marginal}), we obtain
\[
p_h(\mathbf{x}_t,t\mid\mathbf{x}_0,0)
=\frac{p(\mathbf{x}_t,t|\mathbf{x}_0,0)h(\mathbf{x}_t,t)}{h(\mathbf{x}_0,0)}=
\mathcal N
\left(
\mathbf{x}_t;
\bar{\boldsymbol{\mu}}_{0:t}
+
\frac{
e^{-\bar{\theta}_{t:T}}\bar{\sigma}_{0:t}^2
}
{
\hat{\sigma}^2+\bar{\sigma}_{0:T}^2
}
(\hat{\mathbf{x}}-\bar{\boldsymbol{\mu}}_{0:T}),
\frac{
\bar{\sigma}_{0:t}^2
(\hat{\sigma}^2+\bar{\sigma}_{t:T}^2)
}
{
\hat{\sigma}^2+\bar{\sigma}_{0:T}^2
}
\mathbf I
\right).
\]
\end{proof}

\subsection{Proof of Proposition~\ref{prop:generalization}}
\label{sec:prop_generalization}

\propgeneralization*

\begin{proof}
We prove the proposition by substituting the corresponding
hyper-parameter choices into the general reformulated SDDBM dynamics
in Eq.~(\ref{eq:final_sde}). 
Let
$\alpha=0, \beta=1,\gamma=0$.
Recall that 
\begin{align}
\phi_t
    &=
    e^{-\bar{\theta}_{0:t}}
    \left(\hat{\sigma}^2+\bar{\sigma}^2_{t:T}\right)
    +
    \alpha e^{\bar{\theta}_{t:T}}
    \left(\bar{\sigma}^2_{0:T}-\bar{\sigma}^2_{t:T}\right),\label{eq-phi-def-recall}\\
    \hat{\sigma}^2
&=
\frac{
\sigma^2\bar{\sigma}_{0:T}^2
}
{
\bar{\sigma}_{0:T}^2-\sigma^2
}.\label{eq:sigma_hat_sigma_relation}
\end{align}
Moreover, for the GOU process, 
\begin{equation}
    e^{-\bar{\theta}_{0:T}}
    =
    e^{-\bar{\theta}_{0:t}}
    e^{-\bar{\theta}_{t:T}},
    \qquad
    \bar{\sigma}_{0:T}^2
    =
    e^{-2\bar{\theta}_{t:T}}\bar{\sigma}_{0:t}^2
    +
    \bar{\sigma}_{t:T}^2.
    \label{eq:basic_id_generalization}
\end{equation}

By $\alpha=0$ and Eq. (\ref{eq-phi-def-recall}), we have
\begin{equation}
    \phi_t
    =
    e^{-\bar{\theta}_{0:t}}
    \left(
    \hat{\sigma}^2+\bar{\sigma}_{t:T}^2
    \right).
    \label{eq:phi_alpha_zero}
\end{equation}
Substituting \(\alpha=0\) and Eq. (\ref{eq:phi_alpha_zero}) into Eq.~(\ref{eq:eta}), we obtain
\begin{equation}
    \begin{aligned}
    \eta_t^2
    =&
    \frac{g_t^2}{\bar{\sigma}_{0:T}^4}
    \left\{
    \varphi_t
    \left[
    1+
    \frac{
    2e^{-2\bar{\theta}_{t:T}}\bar{\sigma}_{0:t}^2
    }
    {
    \hat{\sigma}^2+\bar{\sigma}_{t:T}^2
    }
    \right]
    +
    \bar{\sigma}_{0:t}^2e^{-2\bar{\theta}_{t:T}}
    \left(
    \sigma^2-\bar{\sigma}_{0:T}^2
    \right)
    \right\}.
    \end{aligned}
    \label{eq:eta_alpha_zero_start}
\end{equation}
By Eq.~(\ref{eq:sigma_hat_sigma_relation}), we obtain
\begin{equation}
    \varphi_t
    =
    \frac{
    \bar{\sigma}_{0:T}^4
    \left(
    \hat{\sigma}^2+\bar{\sigma}_{t:T}^2
    \right)
    }
    {
    \bar{\sigma}_{0:T}^2+\hat{\sigma}^2
    },
    \qquad
    \sigma^2-\bar{\sigma}_{0:T}^2
    =
    -
    \frac{
    \bar{\sigma}_{0:T}^4
    }
    {
    \bar{\sigma}_{0:T}^2+\hat{\sigma}^2
    }.
    \label{eq:varphi_sigma_minus}
\end{equation}
Substituting Eq.~(\ref{eq:varphi_sigma_minus}) into
Eq.~(\ref{eq:eta_alpha_zero_start}) and using Eq.~(\ref{eq:basic_id_generalization}) gives
\begin{equation}
    \begin{aligned}
    \eta_t^2
    &=
    \frac{g_t^2}{\bar{\sigma}_{0:T}^4}
    \left[
    \frac{
    \bar{\sigma}_{0:T}^4
    \left(
    \hat{\sigma}^2+\bar{\sigma}_{t:T}^2
    \right)
    }
    {
    \bar{\sigma}_{0:T}^2+\hat{\sigma}^2
    }
    \left(
    1+
    \frac{
    2e^{-2\bar{\theta}_{t:T}}\bar{\sigma}_{0:t}^2
    }
    {
    \hat{\sigma}^2+\bar{\sigma}_{t:T}^2
    }
    \right)
    -
    \frac{
    e^{-2\bar{\theta}_{t:T}}\bar{\sigma}_{0:t}^2\bar{\sigma}_{0:T}^4
    }
    {
    \bar{\sigma}_{0:T}^2+\hat{\sigma}^2
    }
    \right]
    \\
    &=
    g_t^2
    \frac{
    \hat{\sigma}^2+\bar{\sigma}_{t:T}^2
    +
    e^{-2\bar{\theta}_{t:T}}\bar{\sigma}_{0:t}^2
    }
    {
    \bar{\sigma}_{0:T}^2+\hat{\sigma}^2
    }
    \\
    &=
    g_t^2.
    \end{aligned}
    \label{eq:eta_alpha_zero_gt}
\end{equation}
Therefore, \(\eta_t=g_t\) holds for all cases with
\(\alpha=0\).

We now simplify the drift under the same hyper-parameter choice. 
By Eqs. (\ref{eq:basic_id_generalization}) and (\ref{eq:phi_alpha_zero}),
we have 
\begin{align}\label{eq:phi_alpha_zero_theta}
\frac{e^{-\bar{\theta}_{0:T}}}{\phi_t} = \frac{ e^{-\bar{\theta}_{t:T}}}{
    \hat{\sigma}^2+\bar{\sigma}_{t:T}^2}.
\end{align}
From Eqs.~(\ref{eq:ft_def}),~(\ref{eq:ht_def}), and~(\ref{eq:phi_alpha_zero_theta}), we obtain


\begin{equation}
    {f_t
    =
    -\theta_t
    -g^2_t\frac{ e^{-2\bar{\theta}_{t:T}} }{\left(\hat{\sigma}^2+\bar{\sigma}^2_{t:T}\right)}},
    \label{eq:ft_alpha_zero}
\end{equation}

\begin{equation}
    {m_t
    =
    g_t^2
    \frac{\beta
    e^{-\bar{\theta}_{t:T}}
    }
    {
    \hat{\sigma}^2+\bar{\sigma}_{t:T}^2
    }},
    \label{eq:mt_alpha_zero}
\end{equation}

and
\begin{equation}
    {h_t
    =
    \theta_t
    +
    g_t^2
    \frac{
    e^{-\bar{\theta}_{t:T}}
    \left[ e^{-\bar{\theta}_{t:T}}+(\gamma-1) \right]
    }
    {
    \hat{\sigma}^2+\bar{\sigma}_{t:T}^2
    }}.
    \label{eq:ht_alpha_zero}
\end{equation}

\paragraph{DDBM-VP.}
For DDBM-VP, we choose
\[
\theta_t=\frac12g_t^2,
\qquad
\boldsymbol{\mu}=\mathbf 0,
\qquad
\alpha=0,\quad
\beta=1,\quad
\gamma=0,
\qquad
\hat{\sigma}=0.
\]
Then
\[
\hat{\sigma}^2+\bar{\sigma}_{t:T}^2
=
\bar{\sigma}_{t:T}^2.
\]
Since \(\boldsymbol{\mu}=\mathbf 0\), the term \(h_t\boldsymbol{\mu}\)
vanishes. 
Thus, by Eqs.~(\ref{eq:ft_alpha_zero}) and (\ref{eq:mt_alpha_zero}), the drift in Eq.~(\ref{eq:final_sde}) becomes
\begin{equation}
    \begin{aligned}
    f_t\mathbf{x}_t+m_t\mathbf{x}^{\star}
    &=
    -
    \left(
    \frac12g_t^2
    +
    g_t^2
    \frac{
    e^{-2\bar{\theta}_{t:T}}
    }
    {
    \bar{\sigma}_{t:T}^2
    }
    \right)\mathbf{x}_t
    +
    g_t^2
    \frac{
    e^{-\bar{\theta}_{t:T}}
    }
    {
    \bar{\sigma}_{t:T}^2
    }
    \mathbf{x}^{\star}.
    \end{aligned}
\end{equation}
Together with \(\eta_t=g_t\), this yields
\begin{equation}
   \mathrm{d}\mathbf{x}_t
    =
    \left[
    -
    \left(
    \frac12g_t^2
    +
    g_t^2
    \frac{
    e^{-2\bar{\theta}_{t:T}}
    }
    {
    \bar{\sigma}_{t:T}^2
    }
    \right)\mathbf{x}_t
    +
    g_t^2
    \frac{
    e^{-\bar{\theta}_{t:T}}
    }
    {
    \bar{\sigma}_{t:T}^2
    }
    \mathbf{x}^{\star}
    \right]\mathrm{d} t
    +
    g_t d\mathbf w_t,
\end{equation}
which recovers the DDBM bridge with a VP reference process.

\paragraph{DDBM-VE.}
For DDBM-VE, we choose
\[
\theta_t=0,
\qquad
\boldsymbol{\mu}=\mathbf 0,
\qquad
\alpha=0,\quad
\beta=1,\quad
\gamma=0,
\qquad
\hat{\sigma}=0.
\]
Then
\[
\hat{\sigma}^2+\bar{\sigma}_{t:T}^2
=
\bar{\sigma}_{t:T}^2.
\]
Since \(\theta_t=0\), we have
\[
\bar{\theta}_{t:T}=0,
\qquad
e^{-\bar{\theta}_{t:T}}=1.
\]
Therefore, by Eqs.~(\ref{eq:ft_alpha_zero}) and~(\ref{eq:mt_alpha_zero}),
\[
f_t
=
-
g_t^2
\frac{1}{\bar{\sigma}_{t:T}^2},
\qquad
m_t
=
g_t^2
\frac{1}{\bar{\sigma}_{t:T}^2}.
\]
Again \(h_t\boldsymbol{\mu}=0\), and the drift in Eq.~(\ref{eq:final_sde}) becomes
\[
f_t\mathbf{x}_t+m_t\mathbf{x}^{\star}=g_t^2
\frac{
\mathbf{x}^{\star}-\mathbf{x}_t
}
{
\bar{\sigma}_{t:T}^2
}.
\]
Together with \(\eta_t=g_t\), we obtain
\begin{equation}
   \mathrm{d}\mathbf{x}_t
    =
    g_t^2
    \frac{
    \mathbf{x}^{\star}-\mathbf{x}_t
    }
    {
    \bar{\sigma}_{t:T}^2
    }\mathrm{d} t
    +
    g_t d\mathbf w_t,
\end{equation}
which recovers the DDBM bridge with a VE reference process.

\paragraph{GOUB.}
For GOUB, we choose
\[
\theta_t=\frac{1}{2\lambda^2}g_t^2,
\qquad
\boldsymbol{\mu}=\mathbf{x}^{\star},
\qquad
\alpha=0,\quad
\beta=1,\quad
\gamma=0,
\qquad
\hat{\sigma}=0.
\]
By \(\boldsymbol{\mu}=\mathbf{x}^{\star}\), 
Eqs.~(\ref{eq:ft_alpha_zero}),~(\ref{eq:mt_alpha_zero}) and~(\ref{eq:ht_alpha_zero}),  the drift in Eq.~(\ref{eq:final_sde}) becomes
\begin{equation}\label{eq-drift-mueqxstar}
    \begin{aligned}
    &\quad f_t\mathbf{x}_t+m_t\mathbf{x}^{\star} + h_t \boldsymbol{\mu} =f_t \mathbf{x}_t+(m_t+h_t)\mathbf{x}^{\star} \\
    &= -\left(\theta_t
    +
    g_t^2
    \frac{
    e^{-2\bar{\theta}_{t:T}}
    }
    {
    \hat{\sigma}^2+ \bar{\sigma}_{t:T}^2
    }\right)\mathbf{x}_t + 
    \left( g_t^2
    \frac{
    e^{-\bar{\theta}_{t:T}}
    }
    {\hat{\sigma}^2+
    \bar{\sigma}_{t:T}^2
    }
    +
    \theta_t
    +
    g_t^2
    \frac{
    e^{-\bar{\theta}_{t:T}}
    \left(
    e^{-\bar{\theta}_{t:T}}-1
    \right)
    }
    {\hat{\sigma}^2+
    \bar{\sigma}_{t:T}^2
    }\right)\mathbf{x}^{\star}
    \\&=
    \left(\theta_t
    +
    g_t^2
    \frac{
    e^{-2\bar{\theta}_{t:T}}
    }
    {
    \hat{\sigma}^2+\bar{\sigma}_{t:T}^2
    }\right) \left(
\mathbf{x}^{\star}-\mathbf{x}_t
\right).
    \end{aligned}
\end{equation}
Together with $\hat{\sigma} = 0$ and \(\eta_t=g_t\), we obtain
\begin{equation}
   \mathrm{d}\mathbf{x}_t
    =
    \left(
    \theta_t
    +
    g_t^2
    \frac{
    e^{-2\bar{\theta}_{t:T}}
    }
    {
    \bar{\sigma}_{t:T}^2
    }
    \right)
    \left(
    \mathbf{x}^{\star}-\mathbf{x}_t
    \right)\mathrm{d} t
    +
    g_t d\mathbf w_t,
\end{equation}
which is the GOUB bridge dynamics.

\paragraph{UniDB.}
For UniDB, we choose
\[
\theta_t=\frac{1}{2\lambda^2}g_t^2,
\qquad
\boldsymbol{\mu}=\mathbf{x}^{\star},
\qquad
\alpha=0,\quad
\beta=1,\quad
\gamma=0,
\qquad
\hat{\sigma}^2=\kappa^{-1},
\]
where \(\kappa\) denotes the UniDB terminal penalty coefficient.
Since \(\boldsymbol{\mu}=\mathbf{x}^{\star}\), we also have Eq. (\ref{eq-drift-mueqxstar}).
Together with $\hat{\sigma}^2=\kappa^{-1}$ and \(\eta_t=g_t\), we obtain
\begin{equation}
   \mathrm{d}\mathbf{x}_t
    =
    \left(
    \theta_t
    +
    g_t^2
    \frac{
    e^{-2\bar{\theta}_{t:T}}
    }
    {
    \kappa^{-1}
    +
    \bar{\sigma}_{t:T}^2
    }
    \right)
    \left(
    \mathbf{x}^{\star}-\mathbf{x}_t
    \right)\mathrm{d} t
    +
    g_t d\mathbf w_t,
\end{equation}
which recovers the UniDB dynamics.
\end{proof}

\subsection{Derivation of the training objective }
\label{sec:training_objective_append}

In this section, we derive the training objective for the proposed SDDBMs. The derivation follows the conditional score matching formulation~\citep{song2021scorebased,vincent2011connection}, based on the soft $h$-transformed forward bridge kernel $p_h(\mathbf{x}_t,t \mid \mathbf{x}_0,0)$. We first specify the forward marginals and a marginally consistent Gaussian one-step transition, derive the corresponding posterior, and then formulate the training objective using a noise-parameterized reverse transition.

For notational simplicity, when no ambiguity arises, we omit the explicit
time indices in the transition densities; for example, we write

\[
p_h(\mathbf{x}_t|\mathbf{x}_0):=p_h(\mathbf{x}_t,t|\mathbf{x}_0,0),
\]
and similarly abbreviate
\begin{align*}
p_\theta(\mathbf{x}_{t-1}\mid \mathbf{x}_t)
&:=
p_\theta(\mathbf{x}_{t-1},t-1\mid \mathbf{x}_t,t),\\
p_h(\mathbf{x}_{t-1}\mid \mathbf{x}_0,\mathbf{x}_t)
&:=
p_h(\mathbf{x}_{t-1},t-1\mid
\mathbf{x}_0,0,\mathbf{x}_t,t).
\end{align*}
From Proposition~\ref{prop:ph}, the forward
marginal of the proposed bridge is
\begin{equation}
    p_h(\mathbf{x}_t\mid \mathbf{x}_0)
    =
    \mathcal N
    \left(
    \mathbf{x}_t;
    \bar{\boldsymbol{\mu}}'_{0:t},
    \bar{\sigma}_{0:t}^{\prime 2}\mathbf I
    \right),
    \label{eq:forward_marginal_training}
\end{equation}
where 

\begin{equation}
    \bar{\boldsymbol{\mu}}'_{0:t}
    =
    a_t\mathbf{x}_0
    +
    b_t{\mathbf{x}^\star}
    +
    c_t\boldsymbol{\mu}, \qquad \qquad \qquad \bar{\sigma}_{0:t}'^2 = {
\bar{\sigma}_{0:t}^2
(\hat{\sigma}^2+\bar{\sigma}_{t:T}^2)
}
/\left({
\hat{\sigma}^2+\bar{\sigma}_{0:T}^2
}\right).
\label{eq-muprime-sigmaprime}
\end{equation}
The coefficients $a_t$, $b_t$, and $c_t$ correspond to $\mathbf{x}_0$, $\mathbf{x}^\star$, and $\boldsymbol{\mu}$, respectively, in the mean trajectory derived in Appendix~\ref{sec:x_0}. Their closed-form expressions are
\[
a_t =
\frac{
e^{-\bar{\theta}_{0:t}}
\left(\hat{\sigma}^2+\bar{\sigma}_{t:T}^2\right)
+
\alpha e^{\bar{\theta}_{t:T}}
\left(\bar{\sigma}_{0:T}^2-\bar{\sigma}_{t:T}^2\right)
}{
\hat{\sigma}^2+\bar{\sigma}_{0:T}^2
},
\]
\[
b_t =
\beta
\frac{
e^{\bar{\theta}_{t:T}}
\left(\bar{\sigma}_{0:T}^2-\bar{\sigma}_{t:T}^2\right)
}{
\hat{\sigma}^2+\bar{\sigma}_{0:T}^2
},
\]
and
\[
c_t =
1-
\frac{
e^{-\bar{\theta}_{0:t}}
\left(\hat{\sigma}^2+\bar{\sigma}_{t:T}^2\right)
+
(1-\gamma)e^{\bar{\theta}_{t:T}}
\left(\bar{\sigma}_{0:T}^2-\bar{\sigma}_{t:T}^2\right)
}{
\hat{\sigma}^2+\bar{\sigma}_{0:T}^2
}.
\]

We construct a Gaussian one-step forward transition that is consistent with
the marginals in Eq.~(\ref{eq:forward_marginal_training}). Specifically,
we define
\begin{equation}
\label{eq:forward_transition_training}
\begin{aligned}
& p_h(\mathbf{x}_t\mid \mathbf{x}_{t-1}) \\ 
    =&\mathcal N
    \left(
    \mathbf{x}_t;
    \frac{a_t}{a_{t-1}}\mathbf{x}_{t-1}
    +
    \left(
    b_t-\frac{a_t b_{t-1}}{a_{t-1}}
    \right){\mathbf{x}^\star}
    +
    \left(
    c_t-\frac{a_t c_{t-1}}{a_{t-1}}
    \right)\boldsymbol{\mu},
    \left(
    \bar{\sigma}_{0:t}^{\prime 2}
    -
    \frac{a_t^2}{a_{t-1}^2}
    \bar{\sigma}_{0:t-1}^{\prime 2}
    \right)\mathbf I
    \right).
\end{aligned}
\end{equation}
This transition is valid whenever
\[
\bar{\sigma}_{0:t}^{\prime 2}
-
\frac{a_t^2}{a_{t-1}^2}
\bar{\sigma}_{0:t-1}^{\prime 2}
\ge 0.
\]
It is also consistent with the desired forward marginal. Indeed, if
\[
\mathbf{x}_{t-1}\mid \mathbf{x}_0
\sim
\mathcal N
\left(
\bar{\boldsymbol{\mu}}'_{0:t-1},
\bar{\sigma}_{0:t-1}^{\prime 2}\mathbf I
\right),
\]
then Eq.~(\ref{eq:forward_transition_training}) gives
\begin{equation}
    \begin{aligned}
\mathbb E[\mathbf{x}_t\mid \mathbf{x}_0]&{=
    \mathbb{E}\left[\mathbb{E}[\mathbf{x}_t|\mathbf{x}_0,\mathbf{x}_{t-1}]|\mathbf{x}_0\right]}\\
    &{=\mathbb{E}\left[ \left.\frac{a_t}{a_{t-1}}\mathbf{x}_{t-1}+\left(b_t-\frac{b_{t-1}a_{t}}{a_{t-1}}\right)\mathbf{x}^\star + \left(c_t-\frac{a_tc_{t-1}}{a_{t-1}}\right)\boldsymbol{\mu} \right| \mathbf{x}_0\right]} \\
    &=
    \frac{a_t}{a_{t-1}}\bar{\boldsymbol{\mu}}'_{0:t-1}
    +
    \left(
    b_t-\frac{a_t b_{t-1}}{a_{t-1}}
    \right){\mathbf{x}^\star}
    +
    \left(
    c_t-\frac{a_t c_{t-1}}{a_{t-1}}
    \right)\boldsymbol{\mu}
    \\
    &=
    a_t\mathbf{x}_0+b_t{\mathbf{x}^\star}+c_t\boldsymbol{\mu}
    =
    \bar{\boldsymbol{\mu}}'_{0:t},
    \end{aligned}
    \label{eq:transition_mean_consistency}
\end{equation}
and
\begin{equation}
\begin{aligned}
    \operatorname{Var}(\mathbf{x}_t\mid \mathbf{x}_0)
    &={\mathbb{E} \left[ \operatorname{Var}\left(\mathbf{x}_t|\mathbf{x}_0,\mathbf{x}_{t-1}\right)|\mathbf{x}_0 \right] + \operatorname{Var} \left( \mathbb{E}\left[\mathbf{x}_t|\mathbf{x}_0,\mathbf{x}_{t-1} \right]|\mathbf{x}_0 \right)} \\
    &={\left(\bar{\sigma}'^2_{0:t}-\frac{a^2_t}{a^2_{t-1}} \bar{\sigma}'^2_{0:t-1}\right)\mathbf{I}+ \operatorname{Var}\left( \left.\frac{a_t}{a_{t-1}}\mathbf{x}_{t-1}+\left(b_t-\frac{a_tb_{t-1}}{a_{t-1}}\right)\mathbf{x}^\star+\left(c_t-\frac{a_tc_{t-1}}{a_{t-1}}\right)\boldsymbol{\mu} \right| \mathbf{x}_0\right)} \\
    &=
    \left(
    \bar{\sigma}_{0:t}^{\prime 2}
    -
    \frac{a_t^2}{a_{t-1}^2}
    \bar{\sigma}_{0:t-1}^{\prime 2}
    \right)\mathbf I+\frac{a_t^2}{a_{t-1}^2}
    \bar{\sigma}_{0:t-1}^{\prime 2}\mathbf I
    =
    \bar{\sigma}_{0:t}^{\prime 2}\mathbf I.
    \label{eq:transition_variance_consistency}
\end{aligned}
\end{equation}
Thus, the transition in Eq.~(\ref{eq:forward_transition_training})
preserves the prescribed forward marginal in Eq.~(\ref{eq:forward_marginal_training}).

We now derive the posterior
$p_h(\mathbf{x}_{t-1}\mid \mathbf{x}_0,\mathbf{x}_t)$.
Under the transition in Eq.~(\ref{eq:forward_transition_training}), we can write
\[
\mathbf{x}_t
=
\frac{a_t}{a_{t-1}}\mathbf{x}_{t-1}
+
\left(
b_t-\frac{a_t b_{t-1}}{a_{t-1}}
\right){\mathbf{x}^\star}
+
\left(
c_t-\frac{a_t c_{t-1}}{a_{t-1}}
\right)\boldsymbol{\mu}
+
\boldsymbol{\xi}_t,
\]
where
\[
\boldsymbol{\xi}_t
\sim
\mathcal N
\left(
\mathbf 0,
\left(
\bar{\sigma}_{0:t}^{\prime 2}
-
\frac{a_t^2}{a_{t-1}^2}
\bar{\sigma}_{0:t-1}^{\prime 2}
\right)\mathbf I
\right).
\]
Moreover, \(\boldsymbol{\xi}_t\) is independent of
\(\mathbf{x}_{t-1}\) conditioned on \(\mathbf{x}_0\). 
In addition, Eq.~(\ref{eq:forward_marginal_training}) implies
\[
\mathbf{x}_{t-1}\mid \mathbf{x}_0
\sim
\mathcal N
\left(
\bar{\boldsymbol{\mu}}'_{0:t-1},
\bar{\sigma}_{0:t-1}^{\prime 2}\mathbf I
\right).
\]
Hence,
\((\mathbf{x}_{t-1},\mathbf{x}_t)\mid\mathbf{x}_0\) is
jointly Gaussian. Its conditional mean is
\[
\begin{bmatrix}
\bar{\boldsymbol{\mu}}'_{0:t-1}\\
\bar{\boldsymbol{\mu}}'_{0:t}
\end{bmatrix}.
\]
Moreover, by Eq.~(\ref{eq:forward_marginal_training}),
\[
\operatorname{Var}(\mathbf{x}_{t-1}\mid \mathbf{x}_0)
=
\bar{\sigma}_{0:t-1}^{\prime 2}\mathbf I.
\]
Hence, by Eq.~(\ref{eq:forward_transition_training}), 
\[
\begin{aligned}
\operatorname{Cov}
\left(
\mathbf{x}_{t-1},
\mathbf{x}_t
\mid
\mathbf{x}_0
\right)
&=
\operatorname{Cov}
\left(\left.
\mathbf{x}_{t-1},
\frac{a_t}{a_{t-1}}\mathbf{x}_{t-1}
+
\boldsymbol{\xi}_t
\right|
\mathbf{x}_0
\right)
\\
&=
\frac{a_t}{a_{t-1}}
\operatorname{Var}
\left(
\mathbf{x}_{t-1}
\mid
\mathbf{x}_0
\right)
\\
&=
\frac{a_t}{a_{t-1}}
\bar{\sigma}_{0:t-1}^{\prime 2}\mathbf I.
\end{aligned}
\]
Combined with Eq.~(\ref{eq:forward_marginal_training}), we obtain that the conditional covariance of \((\mathbf{x}_{t-1},\mathbf{x}_t)\mid \mathbf{x}_0\) is
\[
\begin{bmatrix}
\bar{\sigma}_{0:t-1}^{\prime 2}\mathbf I
&
\frac{a_t}{a_{t-1}}
\bar{\sigma}_{0:t-1}^{\prime 2}\mathbf I
\\
\frac{a_t}{a_{t-1}}
\bar{\sigma}_{0:t-1}^{\prime 2}\mathbf I
&
\bar{\sigma}_{0:t}^{\prime 2}\mathbf I
\end{bmatrix}.
\]
Therefore, by the standard conditioning formula for jointly Gaussian
random variables,
\[
p_h(\mathbf{x}_{t-1}\mid \mathbf{x}_0,\mathbf{x}_t)
=
\mathcal N
\left(
\mathbf{x}_{t-1};
\boldsymbol{\mu}_{t-1},
\sigma_{t-1}^2\mathbf I
\right),
\]
where
\[
\boldsymbol{\mu}_{t-1}
=
\bar{\boldsymbol{\mu}}'_{0:t-1}
+
\frac{
a_t\bar{\sigma}_{0:t-1}^{\prime 2}
}
{
a_{t-1}\bar{\sigma}_{0:t}^{\prime 2}
}
\left(
\mathbf{x}_t-\bar{\boldsymbol{\mu}}'_{0:t}
\right),
\]
and
\[
\sigma_{t-1}^2
=
\bar{\sigma}_{0:t-1}^{\prime 2}
-
\frac{
a_t^2
\bar{\sigma}_{0:t-1}^{\prime 4}
}
{
a_{t-1}^2\bar{\sigma}_{0:t}^{\prime 2}
}.
\]

Next, we parameterize the reverse transition as
\begin{equation}
    p_\theta(\mathbf{x}_{t-1}\mid \mathbf{x}_t)
    =
    \mathcal N
    \left(
    \mathbf{x}_{t-1};
    \boldsymbol{\mu}_{\theta,t-1},
    \sigma_{\theta,t-1}^2\mathbf I
    \right).
    \label{eq:model_reverse_training}
\end{equation}
Following the Euler discretization of the reverse-time dynamics, and using
the noise-prediction parameterization
\begin{equation}
    \nabla_{\mathbf{x}_t}\log p_h(\mathbf{x}_t)
    \approx
    -
    \frac{
    \boldsymbol{\epsilon}_\theta(\mathbf{x}_t,\mathbf{x}^\star,t)
    }
    {
    \bar{\sigma}'_{0:t}
    },
    \label{eq:noise_prediction_score}
\end{equation}

we set
\begin{equation}
    \boldsymbol{\mu}_{\theta,t-1}
    =
    \mathbf{x}_t
    -
    \left[
    f_t\mathbf{x}_t
    +
    m_t{\mathbf{x}^\star}
    +
    h_t\boldsymbol{\mu}
    +
    \frac{\eta_t^2}{\bar{\sigma}'_{0:t}}
    \boldsymbol{\epsilon}_\theta(\mathbf{x}_t,\mathbf{x}^\star,t)
    \right],
    \label{eq:model_mean_training}
\end{equation}
and
\begin{equation}
    \sigma_{\theta,t-1}^2=\eta_t^2.
    \label{eq:model_variance_training}
\end{equation}
Here, $\boldsymbol{\epsilon}_\theta(\mathbf{x}_t,\mathbf{x}^\star,t)$ is a neural network parameterized by $\theta$ that predicts the standard Gaussian noise component.
If a discretization step size \(\Delta t\) is used, then the drift term in
Eq.~(\ref{eq:model_mean_training}) should be multiplied by \(\Delta t\),
and the model variance should be set to
\(\sigma_{\theta,t-1}^2=\eta_t^2\Delta t\).

Since $\hat{\sigma}^2>0$ and $\bar{\sigma}^2_{0:t}>0$ for any $t>0$, Eq.~(\ref{eq-muprime-sigmaprime}) implies $\bar{\sigma}'^2_{0:t}>0$ for all $t \in (0,T]$. In particular, at the terminal time, $\bar{\sigma}'^2_{0:T}>0$. Therefore, unlike hard bridge parameterizations whose noise scale vanishes at the terminal boundary, our parameterization does not encounter a division-by-zero singularity as $t$ approaches $T$. This contrasts with the noise-parameterized score used in methods such as GOUB and UniDB, where the endpoint behavior requires special numerical treatment.

Training matches the true posterior $p_h(\mathbf{x}_{t-1}\mid \mathbf{x}_0,\mathbf{x}_t)$ with the parameterized reverse transition $p_{\theta}(\mathbf{x}_{t-1}\mid \mathbf{x}_t)$.
The variational training objective minimizes the KL divergence
between these two distributions:
\begin{equation}
    \mathrm{KL}
    \left(
    p_h(\mathbf{x}_{t-1}\mid\mathbf{x}_0,\mathbf{x}_t)
    \,\|\,
    p_\theta(\mathbf{x}_{t-1}\mid\mathbf{x}_t)
    \right).
    \label{eq:reverse_KL_training}
\end{equation}
Using the KL divergence between two isotropic Gaussian distributions, we
obtain

\begin{equation*}
{
    \begin{aligned}&\mathrm{KL} \left(p_h(\mathbf{x}_{t-1}|\mathbf{x}_0,\mathbf{x}_t)||p_{\theta}(\mathbf{x}_{t-1}|\mathbf{x}_t)\right) \\
        =& \mathbb{E}_{p_h(\mathbf{x}_{t-1}|\mathbf{x}_0,\mathbf{x}_t)} \left[ \log \frac{\frac{1}{(2\pi \sigma^2_{t-1})^{\frac{d}{2}}}\exp\left[{-{||\mathbf{x}_{t-1}-\boldsymbol{\mu}_{t-1}||^2}/{2 \sigma^2_{t-1}}}\right]}{\frac{1}{(2\pi\sigma_{\theta,t-1}^2)^{\frac{d}{2}}}\exp\left[{-{||\mathbf{x}_{t-1}-\boldsymbol{\mu}_{\theta,t-1}||^2}/{2\sigma^2_{\theta,t-1}}}\right]}\right] \\
        =& \mathbb{E}_{p_h(\mathbf{x}_{t-1}|\mathbf{x}_0,\mathbf{x}_t)} \left[d \log \sigma_{\theta,t-1}- d \log \sigma_{t-1}-\frac{||\mathbf{x}_{t-1}-\boldsymbol{\mu}_{t-1}||^2}{2\sigma^2_{t-1}}+ \frac{||\mathbf{x}_{t-1}-\boldsymbol{\mu}_{\theta,t-1}||^2}{2\sigma^2_{\theta,t-1}}\right] \\
        =& d \log \sigma_{\theta,t-1}-d \log \sigma_{t-1}-\frac{d}{2} +\frac{d \sigma^2_{t-1}}{2 \sigma^2_{\theta,t-1}}+\frac{||\boldsymbol{\mu}_{t-1}-\boldsymbol{\mu}_{\theta,t-1}||^2}{2\sigma^2_{\theta,t-1}}.
    \end{aligned}}
\end{equation*}

When \(\sigma_{\theta,t-1}^2\) is fixed and not learned, the first four
terms on the right-hand side of the above equality are independent of the network
parameters. Dropping these constants yields the training objective
\begin{equation}
    \mathcal L
    =
    \mathbb E_{t,\mathbf{x}_0,\mathbf{x}_t}
    \left[
    \frac{
    1
    }
    {
    2\sigma_{\theta,t-1}^2
    }
    \left\|
    \boldsymbol{\mu}_{t-1}
    -
    \boldsymbol{\mu}_{\theta,t-1}
    \right\|^2
    \right].
    \label{eq:training_objective_final}
\end{equation}
In practice, following prior image-to-image
diffusion models~\citep{yue2023image,zhu2025unidb,luo2023image}, we omit the coefficient $\frac{1}{2\sigma^2_{\theta,t-1}}$ and replace the squared \(L_2\) mean-matching term
with an \(L_1\) surrogate to better preserve pixel-level details.
This completes the derivation.

\section{GOU Process}
\label{sec:gou_p}

\begin{theorem}
\label{theorem:gou_p}
    Consider the following GOU process:
    \begin{equation}\label{eq:gou_process}
   \mathrm{d}\mathbf{x}_t= \theta_t \left(\boldsymbol{\mu}-\mathbf{x}_t\right)\mathrm{d} t + g_t \mathrm{d}\mathbf{w}_t
\end{equation}
where $\boldsymbol{\mu}$ is a given state vector, $\theta_t$ denotes a scalar coefficient and $g_t$ represents the diffusion coefficient. We assume the parameterization $g_t^2=2\lambda^2\theta_t$ with $\theta_t>0$, where $\lambda>0$ is a constant. It admits a closed-form analytical solution:

\begin{equation}
    p(\mathbf{x}_t|\mathbf{x}_s)= \mathcal{N} \left(\mathbf{x}_t; \boldsymbol{\mu}+(\mathbf{x}_s-\boldsymbol{\mu})e^{-\bar{\theta}_{s:t}},\frac{g^2_t}{2 \theta_t} \left( 1-e^{-2\bar{\theta}_{s:t}}  \right)\mathbf{I} \right), \,\,\, \bar{\theta}_{s:t}=\int_s^t \theta_\tau \mathrm{d}\tau
\end{equation}

\end{theorem}

\begin{proof}
    We define ${y}_t:=e^{\int_0^t{\theta}_u \,\mathrm{d} u}$, and $\mathbf{z}_t=\mathbf{x}_t y_t$. Thus, with It\^o's formula, we get:
    \begin{equation}
    \begin{aligned}
        \mathrm{d} \mathbf z_t&=\mathrm{d} (\mathbf{x}_t y_t) \\
        &= \mathbf{x}_t \mathrm{d} y_t + y_t\mathrm{d}\mathbf{x}_t \\
        &= \theta_t \mathbf{x}_t y_t\mathrm{d} t +   y_t\left[\theta_t(\boldsymbol{\mu}-\mathbf{x}_t) \mathrm{d} t + g_t \mathrm{d}\mathbf{w}_t\right]\\
        &= \boldsymbol{\mu} \theta_t y_t \mathrm{d} t + g_t y_t \mathrm{d}\mathbf{w}_t
    \end{aligned}
\end{equation}

Integrating from $s$ to $t$, we get:
\begin{equation}
    \begin{aligned}
        \int_s^t  \mathrm{d} \mathbf{z}_u &= \boldsymbol{\mu} \int_s^t \theta_u y_u \,\mathrm{d} u +\int_s^t g_u y_u \,\mathrm{d} \mathbf{w}_u \\
        \mathbf{z}_t&=\mathbf{z}_s + \boldsymbol{\mu} \int_s^t \theta_u y_u \,\mathrm{d} u +\int_s^t g_u y_u \,\mathrm{d} \mathbf{w}_u \\
        \mathbf{x}_t&=\frac{y_s}{y_t} \mathbf{x}_s + \frac{\boldsymbol{\mu}}{y_t} \int_s^t \theta_u y_u \,\mathrm{d} u + \frac{1}{y_t}\int_s^t g_u y_u \,\mathrm{d} \mathbf{w}_u \\
        &= e^{-\int_s^t \theta_u \,\mathrm{d} u} \mathbf{x}_s + \boldsymbol{\mu} \int_s^t e^{-\int_u^t \theta_{\tau}\,\mathrm{d}  \tau} \theta_u \,\mathrm{d} u +\int_s^t e^{-\int_u^t \theta_{\tau}\,\mathrm{d}  \tau} g_u \,\mathrm{d} \mathbf{w}_u\\
        &=e^{-\int_s^t \theta_u \,\mathrm{d} u} \mathbf{x}_s+ \boldsymbol{\mu} \left[1-e^{-\int_s^t \theta_u \,\mathrm{d} u}\right]+ \int_s^t e^{-\int_u^t \theta_{\tau}\,\mathrm{d}  \tau} g_u \,\mathrm{d} \mathbf{w}_u
    \end{aligned}
\end{equation}

It is obvious that the transition kernel is a Gaussian distribution. And we have:
\begin{equation}
    \begin{aligned}
        \int_s^t e^{-\int_u^t \theta_{\tau}\,\mathrm{d}  \tau} g_u \,\mathrm{d} \mathbf{w}_u &\sim \mathcal{N} \left(\mathbf{0}, \int_s^t e^{-2\int_u^t \theta_{\tau}\,\mathrm{d}  \tau} g^2_u \,\mathrm{d} u\, \mathbf{I} \right) \\
        &= \mathcal{N} \left( \mathbf{0}, \lambda^2\int_s^t e^{-2\int_u^t\theta_{\tau}\,\mathrm{d}  \tau} 2 \theta_u \,\mathrm{d}  u \, \mathbf{I}  \right) \\
        &=\mathcal{N} \left(\mathbf{0},\frac{g^2_t}{2\theta_t}\left( 1-e^{-2\bar{\theta}_{s:t}} \right) \mathbf{I}\right)
    \end{aligned}
\end{equation}

Therefore:

\begin{equation}
    p(\mathbf{x}_t \mid \mathbf{x}_s) = \mathcal{N} \left(\mathbf{x}_t; e^{-\bar{\theta}_{s:t}}\mathbf{x}_s+\left(1-e^{-\bar{\theta}_{s:t}}\right)\boldsymbol{\mu},\frac{g_t^2}{2\theta_t} \left(1-e^{-2\bar{\theta}_{s:t}}\right) \mathbf{I} \right).
\end{equation}

This concludes the proof of Theorem~\ref{theorem:gou_p}.

\end{proof}

\section{Experimental Details}
\label{sec:imp_details}


We follow the network architecture and evaluation protocols of
GOUB~\citep{yue2023image} and UniDB~\citep{zhu2025unidb}, with the training configuration specified below. We parameterize the noise estimation network using a standard U-Net architecture~\citep{luo2023image}. The models are optimized using Adam~\citep{DBLP:journals/corr/KingmaB14} with momentum parameters $\beta_1=0.9$ and $\beta_2=0.99$, and a learning rate of $10^{-4}$. During training, we use randomly cropped $128 \times 128$ image patches with a batch size of $8$. For the noise schedule $\theta_t$, we adopt the flipped cosine schedule~\citep{nichol2021improved,luo2023image}:
\begin{equation}
    \theta_t = 1 - \frac{\cos^2 \left( \frac{t/T + s}{1 + s} \frac{\pi}{2} \right)}{\cos^2 \left( \frac{s}{1 + s} \frac{\pi}{2} \right)},
\end{equation}
where the offset $s=0.008$ is applied to ensure a smooth transition near the boundary. 
Following established protocols, we set $\lambda=30/255$ by default. For larger $\sigma$ values in the ablation study, we increase $\lambda$ as needed to satisfy $\sigma^2<\bar{\sigma}_{0:T}^2
$. The best-performing $\sigma$ values lie within the admissible range under the default $\lambda$.
To maintain numerical stability, the terminal coefficient $e^{-\bar{\theta}_{0:T}}$ is truncated at $0.005$ rather than strictly zero.

All models are trained on a single NVIDIA A100 GPU (40GB). For the deraining and inpainting tasks, we train the models for $1.2$ million iterations. For the super-resolution task, the training budget is set to $600,000$ iterations.

During sampling, we initialize the reverse process with $\mathbf{x}_T \approx (b_T+c_T) \mathbf{x^\star}$, omitting the unavailable source-dependent term and the terminal noise. 

For sampling, we use the following reverse process,
\begin{equation}
\begin{aligned}
    \mathrm{d}\mathbf{x}_t&
    =
    \left[
        \mathbf{f}(\mathbf{x}_t,t)
        +
        \zeta
        \eta_t^2
        \frac{\boldsymbol{\epsilon}_{\theta}
        (\mathbf{x}_t,\mathbf{x}^\star,t)}
        {\bar{\sigma}'_{0:t}}
    \right]\mathrm{d}t
    +\eta_t
    \,\mathrm{d}\bar{\mathbf{w}}_t ,\,
    \mathrm{d} \mathbf{x}_t= \left[\mathbf{f}(\mathbf{x}_t,t)+\frac{1}{2}\zeta \eta_t^2 \frac{\boldsymbol{\epsilon}_{\theta}
        (\mathbf{x}_t,\mathbf{x}^\star,t)}
        {\bar{\sigma}'_{0:t}} \right] \mathrm{d} t
\end{aligned}
\end{equation}

When $\zeta=1$, the above formulations recover the standard reverse
SDE and its corresponding probability flow ODE. However,
Eq.~(\ref{eq:model_mean_training}) is derived from a first-order Euler
discretization of the reverse SDE, which introduces discretization
errors. Therefore, directly applying the theoretically derived drift
coefficient may lead to inaccurate score estimation and unstable
sampling trajectories in practice. To compensate for this discrepancy,
we introduce $\zeta$ as an empirical correction factor for the
score-driven drift term. As shown in Table~\ref{tab:param_zeta}, adjusting $\zeta$ provides improved
restoration performance and more stable sampling trajectories compared
with using the theoretical coefficient. 
Therefore, we treat $\zeta$ as an empirical coefficient and set it to
$1.12$, $1.115$, and $1.045$ for deraining, inpainting, and
super-resolution, respectively.

\begin{table*}[ht]
\setlength{\tabcolsep}{2pt}
  \centering
  \vspace{-4mm}
  \caption{Quantitative comparison of different sampling strategies on image restoration tasks. 
Results are reported on DIV2K, Rain100H, and CelebA-HQ $256 \times 256$ datasets for image super-resolution, deraining, and inpainting, respectively. }
  \vskip 0.1in
  \renewcommand{\arraystretch}{1.1}
  \resizebox{\textwidth}{!}{
  \begin{tabular}{|c|cccc|c|cccc|c|cccc|}
    \hline
    \multirow{2}*{\textbf{METHOD}} & \multicolumn{4}{c|}{\textbf{Image Deraining}} & \multirow{2}*{\textbf{METHOD}} & \multicolumn{4}{c|}{\textbf{Image Super-Resolution}} & \multirow{2}*{\textbf{METHOD}} & \multicolumn{4}{c|}{\textbf{Image Inpainting}} \\
    \cline{2-5} \cline{7-10} \cline{12-15}
      & \textbf{PSNR}$\uparrow$ & \textbf{SSIM}$\uparrow$ & \textbf{LPIPS}$\downarrow$ & \textbf{FID}$\downarrow$ & &\textbf{PSNR}$\uparrow$ & \textbf{SSIM}$\uparrow$ & \textbf{LPIPS}$\downarrow$ & \textbf{FID}$\downarrow$ & & \textbf{PSNR}$\uparrow$ & \textbf{SSIM}$\uparrow$ & \textbf{LPIPS}$\downarrow$ & \textbf{FID}$\downarrow$ \\
    \hline

    SDE & 32.38 & 0.9069 & {0.042} & {15.68} &
    SDE & 22.81 & 0.5409 & {0.366} & {33.78} &  SDE  & 29.62 & 0.9084 & {0.034} & {3.86} \\
    Mean-ODE  & {34.96} & {0.9450} & 0.069 & 28.44 &
    Mean-ODE & {28.80} & {0.8141} & 0.314 & 20.30 &  Mean-ODE & {32.00} & \textbf{0.9404} & 0.051 & 11.84 \\   
    \hline 
    SDE-$\zeta$ &{32.94} & {0.9183} & \textbf{0.038} & \textbf{14.21} & SDE-$\zeta$ & 27.22 & 0.7522 & \textbf{0.136} & \textbf{13.70} & SDE-$\zeta$ & {30.30} & {0.9201} & \textbf{0.032} & \textbf{3.48} \\ 
     ODE-$\zeta$  &\textbf{35.06} & \textbf{0.9464} & 0.066 & 25.75 & ODE-$\zeta$ & \textbf{28.90} & \textbf{0.8164} & 0.309 & 19.95 & ODE-$\zeta$ & \textbf{32.22} & {0.9399} & 0.051 & 11.87 \\
     \hline
  \end{tabular}
  }
  \vskip -0.1in
  \label{tab:param_zeta}
\end{table*}

\section{Ablation Study}\label{sec-ablation-study}

In this section, we systematically evaluate the sensitivity of our framework to its core hyperparameters. Recall that the $h$-transform parameters $\alpha$, $\beta$, and $\gamma$ modulate the soft terminal center location, while $\hat{\sigma}^2$ determines its variance spread. We omit $\gamma$ from this analysis because we consistently set $\boldsymbol{\mu} = \textbf{x}^\star$ across all experiments, rendering its independent variation redundant. Furthermore, since $\hat{\sigma}^2$ is explicitly determined by $\sigma$ via Eq.~(\ref{eq-xhat-sigmahat}), our ablation study focuses exclusively on the three essential parameters: $\alpha$, $\beta$, and $\sigma$.



\begin{figure}[ht]
    \centering
    \includegraphics[width=1.0\textwidth]{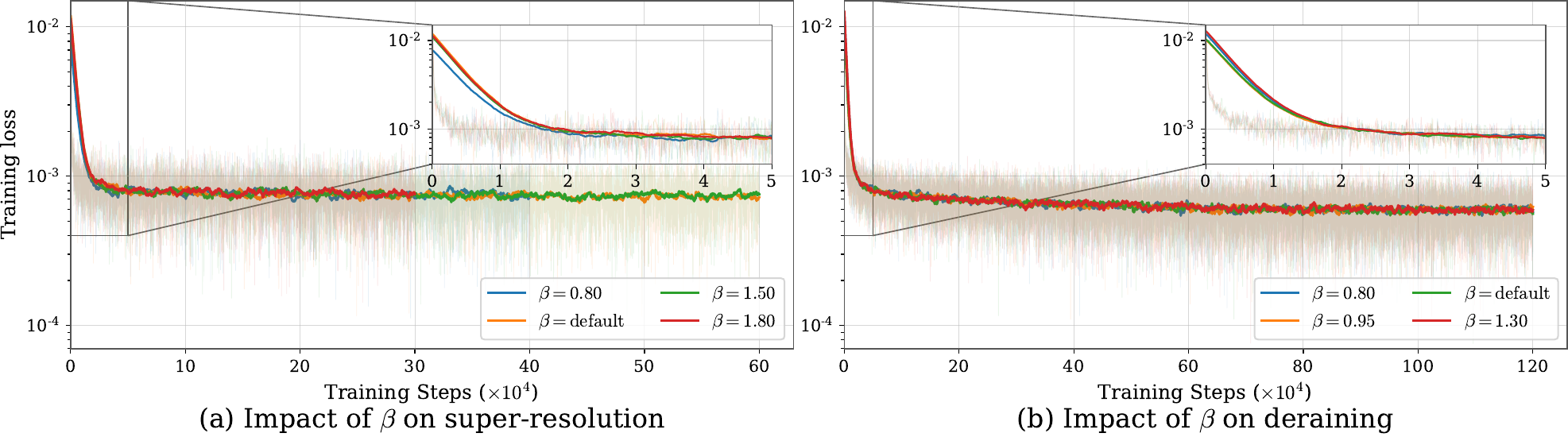}
    \caption{Training loss for different values of $\beta$ on the super-resolution and deraining tasks.}
    \label{fig:param_beta_loss}
\end{figure}


\begin{figure}[ht]
    \centering
    \includegraphics[width=1.0\textwidth]{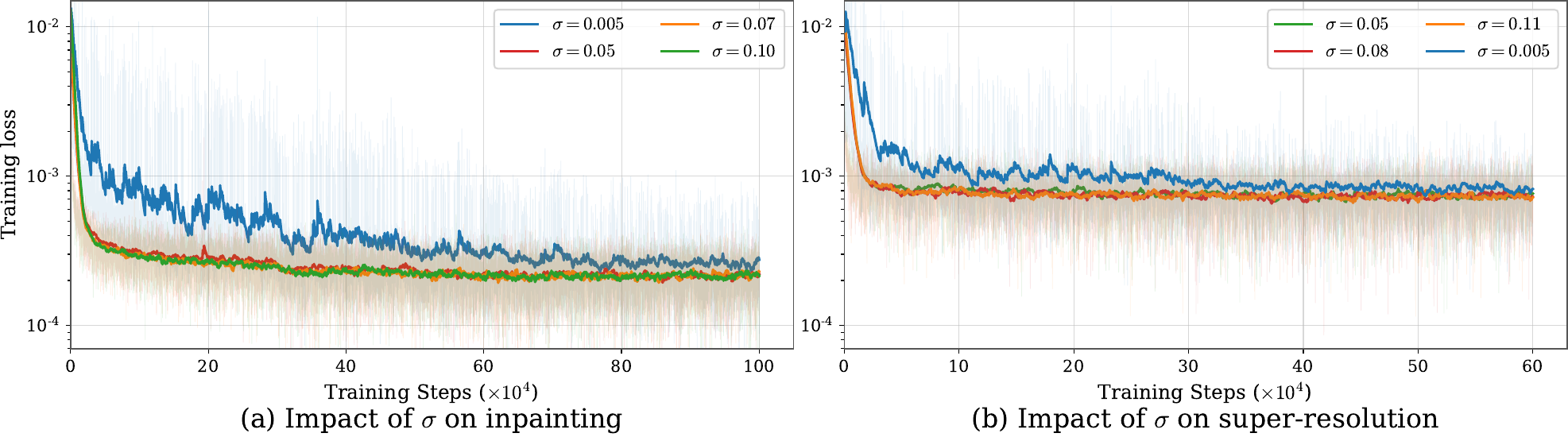}
    \caption{Training loss for different values of $\sigma$ on the inpainting and super-resolution tasks. }
    \label{fig:param_sigma_loss}
\end{figure}

\noindent\textbf{Sensitivity Analysis of $\sigma$:} We first study the effect of the terminal variance parameter $\sigma^2$, which controls the strength of the prescribed Gaussian terminal constraint in SDDBMs. As shown in Figure~\ref{fig:param_alpha_sigma}(a), the performance first improves as $\sigma$ increases, but degrades when $\sigma$ becomes too large. An overly small $\sigma$ makes the terminal distribution highly concentrated, causing the process to approach a hard endpoint bridge and become more susceptible to terminal sampling instability. In contrast, an excessively large $\sigma$ over-relaxes the terminal constraint and weakens target-guided transport. Thus, $\sigma$ governs a key trade-off between sampling stability and target fidelity. Figure~\ref{fig:sigma_deraining} further supports this observation from a training perspective, showing that both extremely small and very large values of $\sigma$ lead to higher training losses, while moderate values stabilize training. These results in Figure~\ref{fig:sigma_deraining} suggest that $\sigma\in[0.05,0.11]$ is a reasonably stable range, while the best performance is typically achieved for $\sigma\in[0.08,0.11]$. Similar trends are also observed on the other two tasks, as shown in Figure~\ref{fig:param_sigma_loss}(a)(b).

\noindent\textbf{Impact of $\alpha$ and $\beta$:} The quantitative results regarding hyperparameter $\alpha$ are depicted in Figure~\ref{fig:param_alpha_sigma}(b). Notably, the restoration performance degrades precipitously as $\alpha$ approaches the critical threshold of $-e^{-\bar{\theta}_{0:T}}\hat{\sigma}^2 / \bar{\sigma}_{0:T}^2$, as established in Lemma~\ref{lem-terminal-relaxion}.
This severe drop empirically validates that nearing this limit triggers the singularity problem, substantially undermining the model's generative efficacy. 
Regarding $\beta$ (Figure~\ref{fig:param_beta_loss}), we observe that different $\beta$ values lead to comparable performance, indicating that the model is robust to the choice of $\beta$ within the tested range. However, a relatively large $\beta$ tends to produce higher training losses during the early stage of optimization and slightly degrades the final performance. Therefore, we recommend choosing $\beta$ values around the default setting\footnote{By default, we set $\beta=1+\frac{\hat{\sigma}^2}{\bar{\sigma}^2_{0:T}}$ and $\gamma=-\frac{\hat{\sigma}^2}{\bar{\sigma}^2_{0:T}}\left(1-e^{-\bar{\theta}_{0:T}}\right)$, yielding $b_T=1$.}, which provides stable training and reliable restoration quality.

\section{Additional Visual Results}

\begin{figure}[ht]
    \centering
    \includegraphics[width=1.0\textwidth]{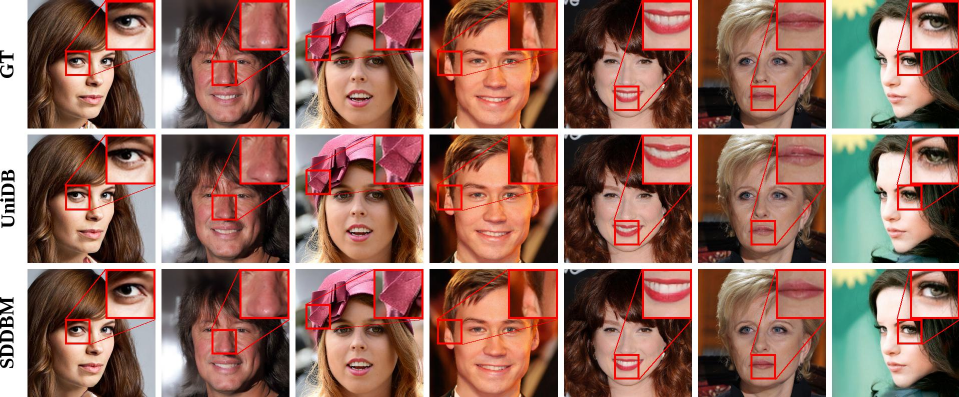}
    \caption{Qualitative comparison between UniDB and SDDBMs on the CelebA-HQ dataset for image inpainting.}
    \label{fig:visual_inpainting}
\end{figure}

\begin{figure}[ht]
    \centering
    \includegraphics[width=1.0\textwidth]{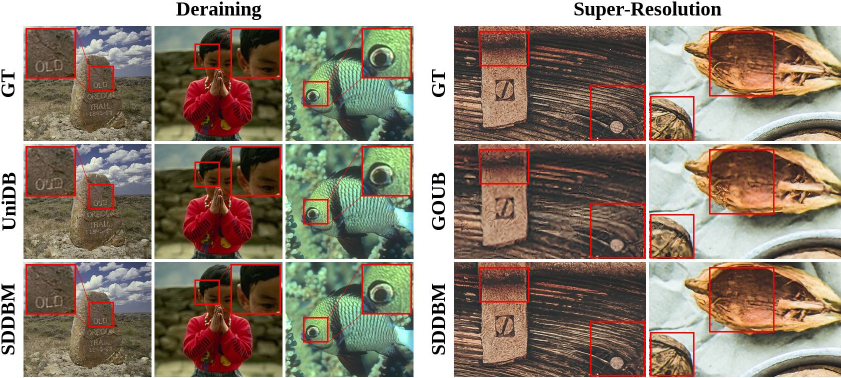}
    \caption{Qualitative comparison between SDDBMs and baseline methods on the Rain100H dataset for image deraining (left) and the DIV2K dataset for image super-resolution (right).}
    \label{fig:visual_dr_sr}
\end{figure}

\begin{figure}[ht]
    \centering
    \includegraphics[width=1.0\textwidth]{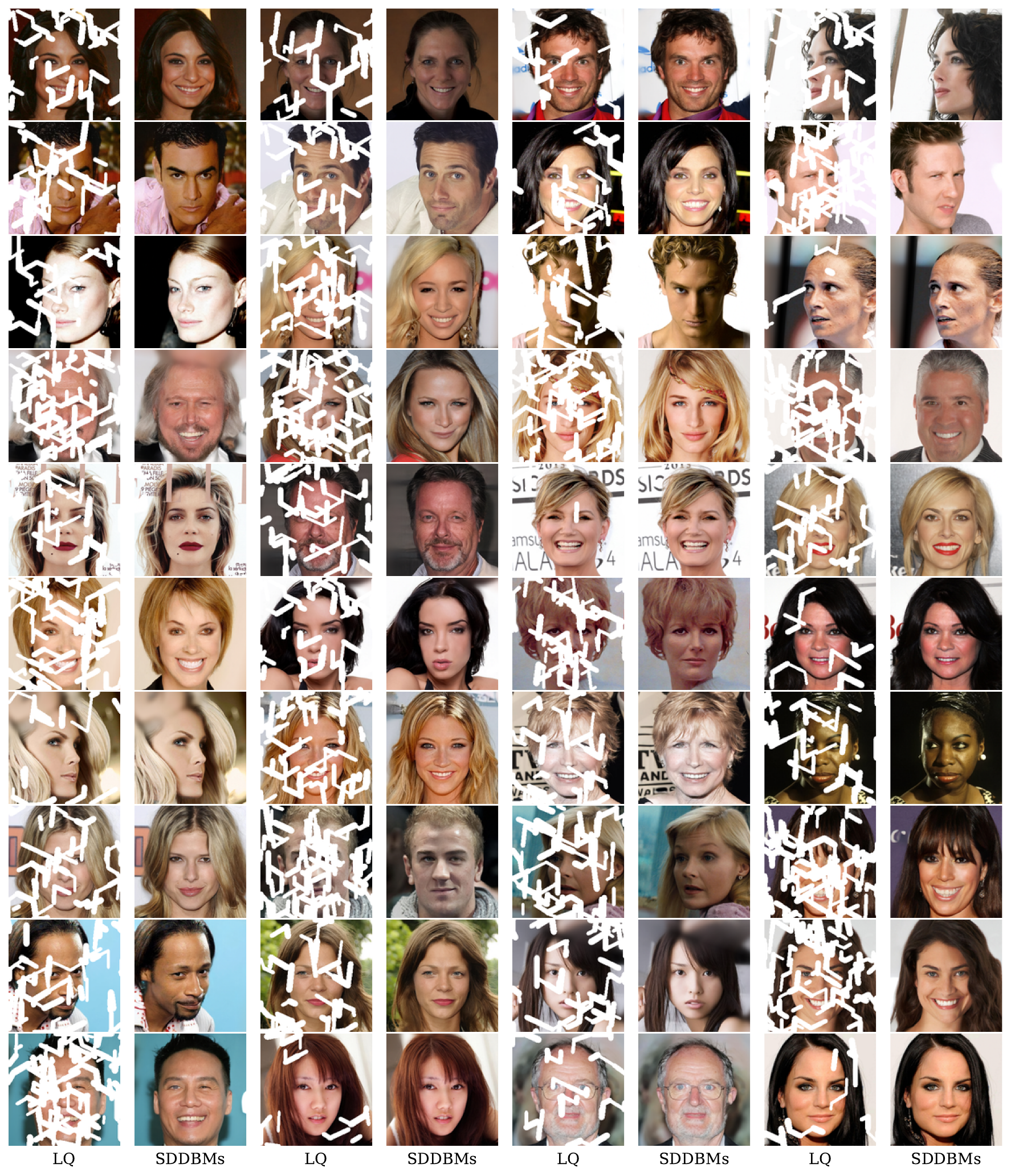}
    \caption{Additional visual results for thin-mask inpainting on the CelebA-HQ dataset.}
    \label{fig:visual_inpainting_appendix}
\end{figure}
\clearpage

\begin{figure}[ht]
    \centering
    \includegraphics[width=1.0\textwidth]{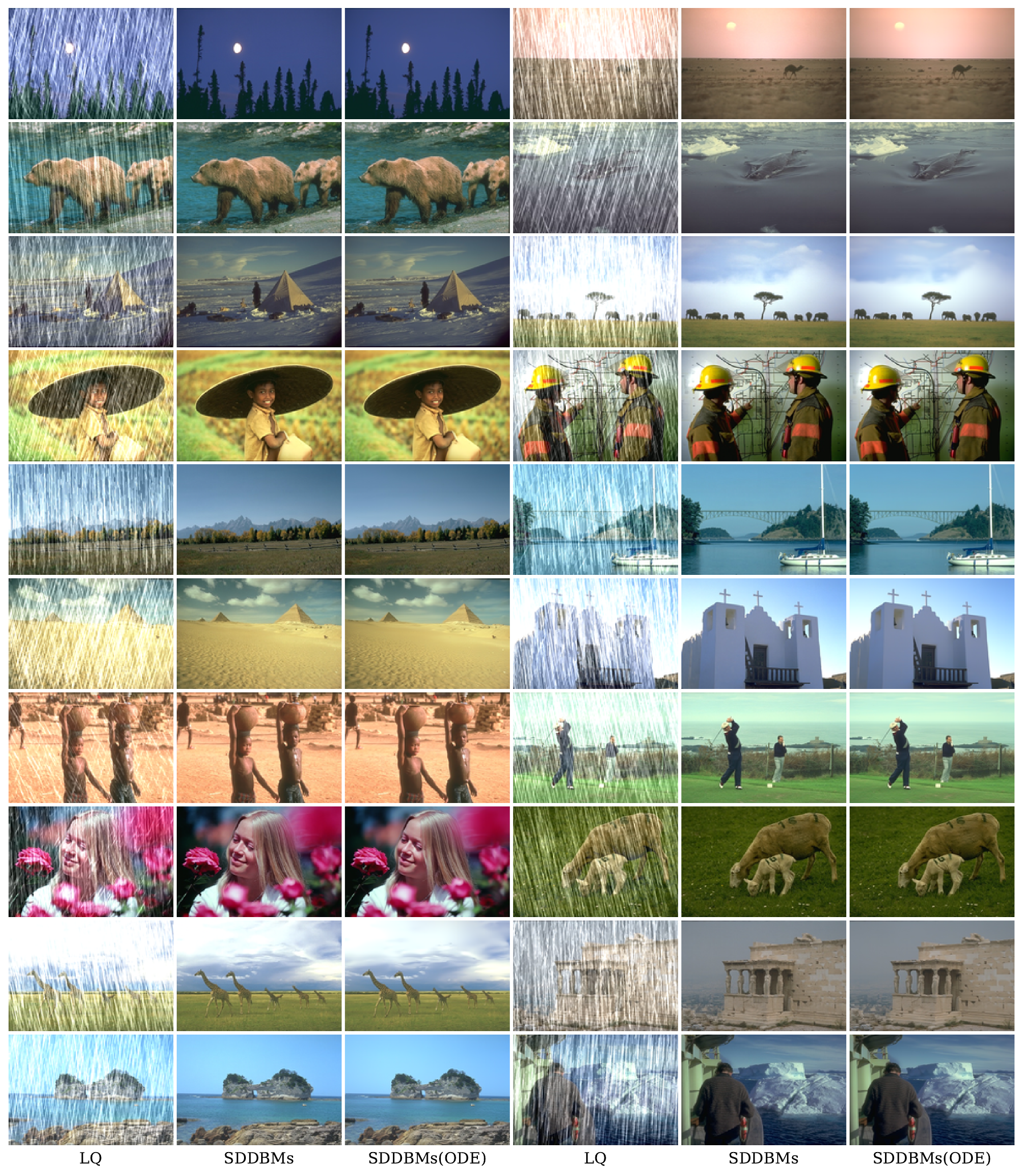}
    \caption{Additional visual results for deraining on the Rain100H dataset.}
    \label{fig:visual_deraining_appendix_1}
\end{figure}
\clearpage

\begin{figure}[ht]
    \centering
    \includegraphics[width=1.0\textwidth]{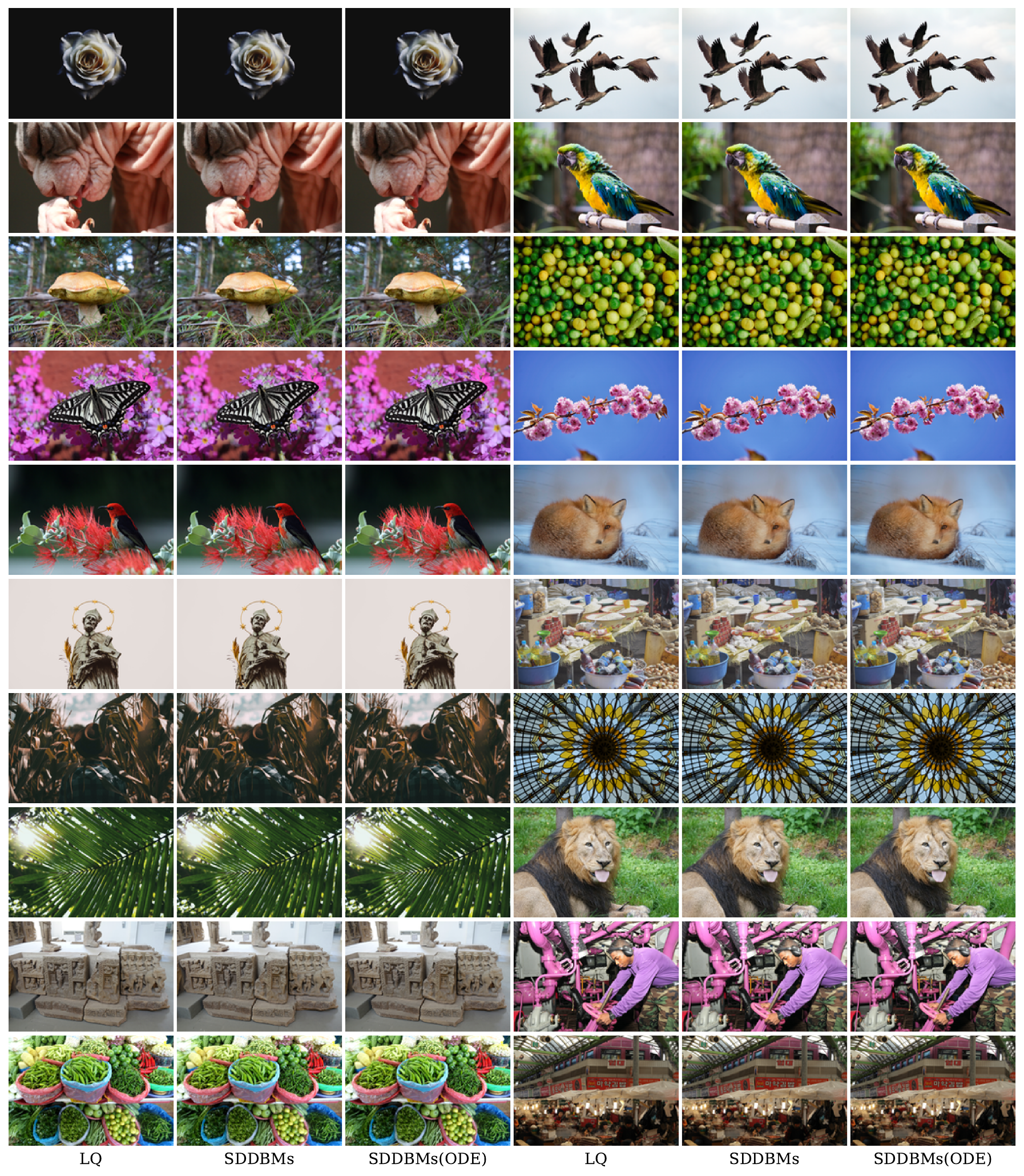}
    \caption{Additional visual results for super-resolution on the DIV2K dataset.}
    \label{fig:visual_deraining_appendix_2}
\end{figure}
\clearpage

\end{document}